\documentclass[10pt]{article} 
\usepackage[preprint]{rlc}

\usepackage{amssymb}            
\usepackage{mathtools}          
\usepackage{mathrsfs}           
\mathtoolsset{showonlyrefs}     
\usepackage{graphicx}           
\usepackage{subcaption}         
\usepackage[space]{grffile}     
\usepackage{url}                

\usepackage{booktabs}%
\usepackage{algorithm}%
\usepackage{algorithmicx}%
\usepackage{algpseudocode}%
\usepackage{listings}%
\usepackage{siunitx}
\usepackage{longtable}
\usepackage{enumitem}

\usepackage{mymacros}

\usepackage{amsmath,amsfonts,bm}

\def\eqref#1{equation~\ref{#1}}

\def\1{\bm{1}}

\def\vtheta{{\bm{\theta}}}

\DeclareMathAlphabet{\mathsfit}{\encodingdefault}{\sfdefault}{m}{sl}
\SetMathAlphabet{\mathsfit}{bold}{\encodingdefault}{\sfdefault}{bx}{n}

\newcommand{\E}{\mathop{\mathbb{E}}}

\newcommand{\Var}{\mathop{\mathbb{V}\mathrm{ar}}}

\newcommand{\Cov}{\mathop{\mathbb{C}\mathrm{ov}}}

\DeclareMathOperator*{\argmax}{arg\,max}
\DeclareMathOperator*{\argmin}{arg\,min}

\usepackage{thm-restate}

\newtheorem{assumption}{Assumption}
\newtheorem{proposition}{Proposition}

\title{Policy Gradient with Active Importance Sampling}

\author{Matteo Papini\thanks{Politecnico di Milano, Milan, Italy. Contact: matteo.papini@polimi.it} \\
    \And
    Giorgio Manganini\thanks{Gran Sasso Science Institute, L'Aquila, Italy}
    \And
    Alberto Maria Metelli\footnotemark[1]  \\
    \And
    Marcello Restelli\footnotemark[1]
    }

\begin{document}

\maketitle

\begin{abstract}
Importance sampling (IS) represents a fundamental technique for a large surge of off-policy reinforcement learning approaches. Policy gradient (PG) methods, in particular, significantly benefit from IS, enabling the effective reuse of previously collected samples, thus increasing sample efficiency. However, classically, IS is employed in RL as a passive tool for re-weighting historical samples. However, the statistical community employs IS as an active tool combined with the use of behavioral distributions that allow the reduction of the estimate variance even below the sample mean one. In this paper, we focus on this second setting by addressing the behavioral policy optimization (BPO) problem. We look for the best behavioral policy from which to collect samples to reduce the policy gradient variance as much as possible. We provide an iterative algorithm that alternates between the cross-entropy estimation of the minimum-variance behavioral policy and the actual policy optimization, leveraging on defensive IS. We theoretically analyze such an algorithm, showing that it enjoys a convergence rate of order $O(\epsilon^{-4})$ to a stationary point, but depending on a more convenient variance term w.r.t. standard PG methods. We then provide a practical version that is numerically validated, showing the advantages in the policy gradient estimation variance and on the learning speed.
\end{abstract}


\section{Introduction}\label{sec1}

\emph{Policy gradient}~\citep[PG,][]{PetersS06} algorithms represent a large class of \emph{reinforcement learning}~\citep[RL,][]{sutton2018reinforcement} approaches that are particularly suitable to address complex control problems thanks to their ability to deal with continuous state and action spaces natively. PG methods address the RL problem by considering a parametric control \emph{policy} $\pi_{\vtheta}$ and formulate the learning process as a particular stochastic optimization problem by updating the policy parameters $\vtheta$ in the ascent direction of the policy gradient. Clearly, the policy gradient needs to be estimated from samples, making the accuracy of such an estimate crucial for the actual performance of the PG approaches~\citep{ZhaoHNS11,PapiniPR22}. 

In this direction, a significant line of research is represented by the approach to sample reuse. Borrowing the techniques from the statistical simulation community, \emph{importance sampling}~\citep[IS,][]{mcbook} has been imported to the PG methods. The majority of the approaches that apply IS to PG methods are based on the idea of reweighting the data collected in the past (i.e., with \emph{behavioral policies}) proportionally to the probability of being generated by the current policy (i.e., \emph{target policy}), whose gradient needs to be estimated~\citep[e.g.,][]{ThomasTG15,MetelliPFR18}. Theoretical results about the advantages in terms of variance reduction have been provided in~\cite{MetelliPMR20}. However, these approaches can be considered \emph{passive} since the focus is on reusing in the most effective way the sample collected in the past without considering the possibility of \emph{choosing} the behavioral policy to improve the estimation of the gradient of the current target policy. 

Indeed, this is the main use of IS for in the Monte Carlo simulation community, where this technique takes an \emph{active} role. Specifically, in these scenarios, the objective is to find the best behavioral policy from which to collect samples in order to reduce the estimate variance as much as possible. It can be proved that under specific assumptions on the random variable whose expectation is to be estimated, such off-policy variance can be reduced even below that of the standard sample mean estimate~\cite{mcbook}. Although this line represents an appealing direction within a class of approaches (like RL) that suffer from an inherent sample inefficiency, the community has not deeply studied this direction.

\textbf{Original Contributions}~~In this paper, we focus on the active role of IS in the PG family of RL algorithms. Specifically, we investigate if we can actively learn the behavioral policy from which to collect samples in order to control the variance of the PG estimator effectively. We call this problem \emph{behavioral policy optimization} (BPO). The contributions of the paper can be stated as follows:
\begin{itemize}[leftmargin=*,topsep=-1pt,noitemsep]
    \item We formulate the BPO problem as finding the behavioral policy that minimizes the variance of the off-policy gradient estimate of a given target policy. After showing that this optimization problem allows for a closed-form solution under restrictive conditions, we introduce an approach for estimating such a behavioral policy based on cross-entropy minimization (Section~\ref{sec:bpo}).
    \item We provide a theoretical analysis of a principled algorithm that alternates two phases: behavioral policy learning based on cross-entropy and actual performance optimization based on the off-policy gradient update. We show that a careful sample partition between the two phases allows for achieving convergence rates of order $O(\epsilon^{-4})$ but depending on a more convenient variance term compared to standard REINFORCE (Section~\ref{sec:theory}).
    \item We provide a practical version of the analyzed algorithm that uses all the samples collected. Then, we empirically evaluate such an algorithm, showing a significant reduction in the variance of the gradient estimate that translates into a faster learning curve (Section~\ref{sec:numerical}).
\end{itemize}

The proofs of all the results reported in the main paper can be found in Appendix~\ref{app:proofs}.

\section{Preliminaries}

\textbf{Notation}~~Let $n \in \Nat$, we denote with $[n] \coloneqq \{1,\dots,n\}$. For a measurable set $\Xs$, we denote with $\Delta^{\Xs}$ the set of probability measures over $\Xs$. Let $P,Q \in \Delta^{\Xs}$ be two probability measures such that $P \ll Q$, that is, $P$ is absolutely continuous with respect to $Q$. When the reference measure $\lambda$ is clear from the context (Lebesgue measure for continuous $\mathcal{X}$ and counting measure for discrete $\mathcal{X}$, respectively), we use $p$ to denote the Radon-Nikodym derivative $\de P / \de \lambda$ (density and mass function, respectively) and $\int_{\mathcal{X}} \cdot\,\de x$ to denote integration with respect to $\lambda$ (Lebesgue integral and summation, respectively). We define the KL-divergence $D_{\text{KL}}$ and the chi-square divergence $\chi^2$ as:
\begin{align}
    D_{\text{KL}}(P\| Q) \coloneqq\int_{\Xs} p(x) \log \left(\frac{p(x)}{q(x)}\right) \de x, \quad  \chi^2(P\|Q) \coloneqq \int_{\Xs} \frac{(p(x)-q(x))^2}{q(x)} \de x.
\end{align}

\textbf{Markov Decision Processes}~~A discounted Markov decision problem~\citep[MDP,][]{puterman2014markov} is defined as a 6-tuple $(\Ss,\As,P,R,\mu_0,\gamma)$, where $\Ss$ is the measurable state space, $\As$ is the measurable action space, $P :\SAs \rightarrow \Delta^{\Ss}$ is the transition model defining for every $(s,a) \in \SAs$ the probability distribution of the next state $s' \sim P(\cdot|s,a)$, $R : \SAs \rightarrow [-R_{\max},R_{\max}]$ is the reward function $R(s,a)$ when performing action $a$ in state $s$, uniformly bounded by $R_{\max} < +\infty$ defining the reward $R(s,a)$ obtained when playing action $a$ in state $s$, $\mu_0 \in \Delta^{\Ss}$ is the initial-state distribution prescribing the state at which interaction begins, $s_0 \sim \mu_0$, and $\gamma \in [0,1]$ is the discount factor.

\textbf{Actor-only Policy Gradient}~~We consider an agent whose behavior is described by a parametric policy $\pi_{\vtheta} : \Ss \rightarrow \Delta^{\As}$ where $\vtheta \in \vTheta$ is the parameter belonging to the parameter space $\vTheta \subseteq \Reals^d$, assumed to be convex. In this setting, the agent's goal consists of maximizing the expected return:
\begin{align*}
    \vtheta^* \in \argmax_{\vtheta \in \vTheta} J(\vtheta) \coloneqq \E_{\vtau \sim p_\vtheta} \left[ R(\vtau) \right], \qquad \text{where} \qquad R(\vtau) \coloneqq \sum_{t=0}^{T-1} \gamma^t R(s_t,a_t),
\end{align*}
and $\vtau = (s_0,a_0,\dots,s_{T-1},a_{T-1}) \in \bm{\mathcal{T}}$ is the trajectory whose probability density function is given by $p_{\vtheta}(\vtau) = \mu_0(s_0)\prod_{t=0}^{T-1} \pi_\vtheta(a_t|s_t) P(s_{t+1}|s_t,a_t)$, $T$ is the trajectory length, and $\bm{\mathcal{T}} = (\SAs)^T$ is the trajectory set.\footnote{For a sufficiently large length, namely $T\ge (1-\gamma)^{-1}\log\left(\epsilon^{-1}R_{\max}(1-\gamma)^{-1}\right)$, the finite-horizon $\gamma$-discounted expected return is $\epsilon$-close to its infinite-horizon counterpart~\citep{kearns2002near}. For this reason, we will use the two interchangeably, and just make sure $T\simeq (1-\gamma)^{-1}$ in our simulations.}
If $\pi_{\vtheta}$ is differentiable in $\vtheta$, we can express the \emph{policy gradient}~\citep{williams1992simple}, that is the gradient of the expected return $J(\vtheta)$ with respect to $\vtheta$:
\begin{align*}
    \nabla J(\vtheta) =  \E_{\vtau \sim p_\vtheta} \left[ \nabla \log p_{\vtheta}(\vtau) R(\vtau) \right].
\end{align*}
Actor-only methods~\citep{PetersS06} perform learning by updating the policy parameters in the direction of the ascending policy gradient $\vtheta \gets \vtheta + \alpha  \nabla J(\vtheta) $, where $\alpha > 0$ is the step size. 

\textbf{On-policy gradient estimators}~~The policy gradient $\nabla J(\vtheta)$ needs to be estimated from a set of collected trajectories. If the trajectories $\mathcal{D}_{\text{on}} = \{\vtau_i\}_{i\in [n]}$ are collected with the same policy $\pi_\vtheta$ of 
 which we seek to estimate the policy gradient, we speak of \emph{on-policy} gradient estimation:
\begin{align}
    \widehat{\nabla} J(\vtheta;\mathcal{D}_{\text{on}}) = \frac{1}{n} \sum_{i=1}^n \mathbf{g}_\vtheta(\vtau_i), \qquad  \vtau_i \sim p_{\vtheta}, \quad \forall i \in [n],
\end{align}
where $\mathbf{g}_\vtheta(\vtau)$ is a single-trajectory estimator of the policy gradient. Classical unbiased estimators include: REINFORCE~\citep{williams1992simple}
    where $\mathbf{g}_\vtheta^{\text{R}}(\vtau) =( \sum_{t=0}^{T-1} \nabla \log  \pi_{\vtheta}(a_t|s_t) ) R(\vtau)$ and G(PO)MPD~\citep{baxter2001infinite} where $\mathbf{g}_\vtheta^{\text{G}}(\vtau) =\sum_{t=0}^{T-1} \gamma^t R(s_t,a_t) \sum_{l=0}^t \nabla \log  \pi_{\vtheta}(a_l|s_l) $.

\textbf{Off-policy gradient estimators with Single behavioral policy}~~When, instead, we seek to estimate the policy gradient $\nabla J(\vtheta)$ of a \emph{target} policy $\pi_{\vtheta}$ having collected $n$ trajectories $\mathcal{D}_{\text{off}} = \{\vtau_i\}_{i\in[n]}$ with a different \emph{behavioral} policy $\pi_{\vtheta^b}$, under the assumption that $\pi_{\vtheta}(\cdot|s) \ll \pi_{\vtheta^b}(\cdot|s)$ for every $s \in \Ss$, we speak of \emph{(single) off-policy} gradient estimation:\footnote{ if  dataset $\mathcal{D}_{\text{off}}$ is made of just one trajectory $\vtau$, with little abuse of notation, we denote the estimator by $\widehat{\nabla} J(\vtheta;\vtau)$.}
\begin{align} \label{eq:off_grad}
    \widehat{\nabla} J(\vtheta;\mathcal{D}_{\text{off}}) = \frac{1}{n} \sum_{i=1}^n\frac{p_\vtheta(\vtau)}{p_{\vtheta^b}(\vtau)}(\vtau_i)\mathbf{g}_\vtheta(\vtau_i), \qquad  \vtau_i \sim p_{\vtheta^b}, \quad \forall i \in [n],
\end{align}
where $\frac{p_\vtheta(\vtau)}{p_{\vtheta^b}(\vtau)}$ is the trajectory \emph{(simple) importance weight}~\citep{mcbook}, defined as:
\begin{align}
   \frac{p_\vtheta(\vtau)}{p_{\vtheta^b}(\vtau)} = \prod_{t=0}^{T-1} \frac{\pi_\vtheta(a_t|s_t)}{\pi_{\vtheta^b}(a_t|s_t)}. 
\end{align}

\textbf{Off-policy gradient estimators with Multiple behavioral policies}~~It is possible to extend these estimators to the case in which trajectories are collected from multiple $m\in\Nat$ behavioral policies parameters $\{\vtheta^b_j\}_{j \in [m]}$. In such a case, for every  $j \in [m]$, we have collected $n_j$ trajectories $\{\vtau_{ij}\}_{i\in [n_j]}$ from the behavioral policy $\pi_{\vtheta_j^b}$ and such that $ \beta_j(\cdot)\pi_{\vtheta}(\cdot|s) \ll \pi_{\vtheta_j^b}(\cdot|s)$ for every $s \in \mathcal{S}$, we speak of \emph{multiple off-policy} gradient estimation:
\begin{align}
     \widehat{\nabla} J(\vtheta;\mathcal{D}_{\text{off}};\bm{\beta}) = \sum_{j=1}^m \frac{1}{n_j}  \sum_{i=1}^{n_j} \beta_j(\vtau_{ij}) \frac{p_{\vtheta}(\vtau_{ij})}{p_{\vtheta^b_j}(\vtau_{ij})} \mathbf{g}_\vtheta(\vtau_{ij}), \quad  \vtau_{ij} \sim p_{\vtheta^b_j}, \quad \forall i \in [n_j], \; \forall j \in [m],
\end{align}
where $\mathcal{D}_{\text{off}}=\{\{\vtau_{ij}\}_{i\in [n_j]}\}_{j \in [m]}$ and $\beta_j(\vtau) \ge 0$ for every $j \in [m]$ and $\sum_{j=1}^m \beta_j(\vtau)=1$ for every trajectory $\vtau \in \bm{\mathcal{T}}$ is a \emph{partition of the unity}. A common choice for the coefficients $\beta_j$ which enjoys desirable theoretical properties is the \emph{balance heuristic}~\citep[BH,][]{VeachG95}:
\begin{align}
    \beta_j^{\text{BH}}(\vtau) \coloneqq \frac{n_j p_{\vtheta^b_j}(\vtau)}{\sum_{k=1}^m n_k p_{\vtheta^b_k}(\vtau)} = \frac{n_j \prod_{t=0}^{T-1}\pi_{\vtheta^b_j}(a_t|s_t)}{\sum_{k=1}^m n_k \prod_{t=0}^{T-1}\pi_{\vtheta^b_k}(a_t|s_t)}.
\end{align}
The resulting estimator becomes:
\begin{align}
     \widehat{\nabla} J(\vtheta;\mathcal{D}_{\text{off}}) = \frac{1}{n} \sum_{j=1}^m  \sum_{i=1}^{n_j}   {\frac{p_{\vtheta}(\vtau_{ij})}{\sum_{k=1}^m \frac{n_j}{n} p_{\vtheta_k^b}(\vtau_{ij})} } \mathbf{g}_\vtheta(\vtau_{ij}), \quad  \vtau_{ij} \sim p_{\vtheta^b_j}, \quad \forall i \in [n_j], \; \forall j \in [m],
\end{align}
where $n = \sum_{j=1}^m n_j$ is the total number of trajectories. The \emph{(multiple) importance weight} can be interpreted as the (single) importance weight having as a behavioral distribution the mixture of the $m$ behavioral distributions with weights $\frac{n_j}{n}$, i.e., $\Phi_m \coloneqq \sum_{k=1}^m \frac{n_j}{n} p_{\vtheta_k^b}$~\citep{MetelliPMR20}.

When the set of behavioral policy parameters contains the target policy parameter $\vtheta$ too, we speak of \emph{defensive (multiple) off-policy} gradient estimation~\cite{mcbook}. In such a case, the importance weight is guaranteed to be bounded.


\section{Behavioral Policy Optimization}\label{sec:bpo}
In this section, we introduce the \emph{behavioral policy optimization} (BPO) problem we aim to solve in this paper. The BPO problem consists in finding the \quotes{best behavioral policy} $\pi_{\vtheta^b}$ to be used for collecting the trajectories $\vtau \sim p_{\vtheta^b}$ for estimating the policy gradient $\widehat{\nabla} J(\vtheta;\vtau)$ of the target policy $\pi_{\vtheta}$.
We formalize the notion of \quotes{best behavioral policy} as the one that minimizes the trace of the covariance matrix of the off-policy gradient estimator $\widehat{\nabla} J(\vtheta;\vtau) $ where $\vtau \sim p_{\vtheta^b}$ (that we will refer to as \emph{gradient variance}) induced by the candidate behavioral policy $\pi_{\vtheta^b}$:\footnote{In the following, we will continue employing the policy gradient notation, although the presented result hold for the estimation of the expected value of a general vector-valued function.}
\begin{align}\label{eq:opt}
    p_{*,\vtheta} \in \argmin_{p_{\vtheta^b} \,:\, \vtheta^b \in \vTheta} \Var_{\vtau \sim p_{\vtheta^b}}\left[\widehat{\nabla} J(\vtheta;\vtau)\right] \coloneqq \E_{\vtau \sim p_{\vtheta^b}} \left[ \left\| \widehat{\nabla} J(\vtheta;\vtau) - \nabla J(\vtheta) \right\|^2_2 \right].
\end{align}
The optimization problem of Equation~(\ref{eq:opt}) can be challenging since it involves a minimization over the parameter space $\vTheta$, which can determine, in general, a non-convex optimization problem. In Section~\ref{sec:closedForm}, we show that when extending the optimization over the full set of distributions over the trajectory space $\bm{\mathcal{T}}$, we can solve the BPO problem in closed form. In Section~\ref{sec:XEntropy}, we illustrate how the closed-form solution can be employed to learn a policy that induces a trajectory distribution representable within the policy parameters space $\vTheta$ approximately close to the best one.


\subsection{Closed-form solution}\label{sec:closedForm}
In this section, we study the solution of the problem of Equation (\ref{eq:opt}) when no restriction to the representable trajectory distributions is enforced. Although this assumption is not realistic from the policy gradient perspective, given the fact that the transition model of the environment is not under control and the policy space might be constrained to the specific parametrization $\vtheta \in \vTheta$, it represents an important preliminary step for obtaining a practical algorithm. The following result provides a closed-form solution to the BPO problem.

\begin{restatable}[]{theorem}{optimalQ}\label{thr:optimalQ}
Let $\vtheta \in \vTheta$ and $\mathbf{g}_{\vtheta} : \bm{\mathcal{T}} \rightarrow \Reals^d$ be the single-trajectory gradient estimator used to compute $\widehat{\nabla} J(\vtheta; \vtau)$. The solution $p_{*,\vtheta} \in \Delta^{\bm{\mathcal{T}}}$ to the BPO problem (Equation~\ref{eq:opt}) is given by:
\begin{align}\label{eq:pStar}
    p_{*,\vtheta} (\vtau) = \frac{p_{\vtheta}(\vtau) \| \mathbf{g}_{\vtheta}(\vtau) \|_2}{\int_{\bm{\mathcal{T}}} p_{\vtheta}(\vtau) \| \mathbf{g}_{\vtheta}(\vtau) \|_2 \de \vtau}.
\end{align} 
The optimal value of Equation~(\ref{eq:opt}) is given by:
\begin{align}\label{eq:off}
     \Var_{\vtau \sim  p_{*,\vtheta}}\left[\widehat{\nabla} J(\vtheta; \vtau) \right]  = \E_{\vtau \sim p_{\vtheta}} \left[ \| \mathbf{g}_{\vtheta}(\vtau) \|_2 \right]^2 - \left\| \nabla J(\vtheta) \right\|_2^2.
\end{align}
\end{restatable}

It is worth comparing the result of Equation~(\ref{eq:off}) with the variance of the on-policy gradient estimator that can be easily computed from Equation~(\ref{eq:opt}):
\begin{align}\label{eq:on}
    \Var_{\vtau \sim p_{\vtheta}}\left[\widehat{\nabla} J(\vtheta; \vtau) \right]  = \E_{\vtau \sim p_{\vtheta}} \left[ \| \mathbf{g}_{\vtheta}(\vtau) \|_2^2 \right] - \left\| \nabla J(\vtheta) \right\|_2^2.
\end{align}
Although the subtracted term $\left\| \nabla J(\vtheta) \right\|_2^2$ is the same in (\ref{eq:on}) and (\ref{eq:off}), the first one presents some differences. Indeed, in Equation~(\ref{eq:on}) we have an \emph{expectation of the squared $L_2$-norm} of the single-trajectory gradient estimator, \ie $\E_{\vtau \sim p_{\vtheta}} \left[ \| \mathbf{g}_{\vtheta}(\vtau) \|_2^2 \right]$, whereas in Equation~(\ref{eq:off}), we have the \emph{squared expectation of the $L_2$-norm} of the single-trajectory gradient estimator, \ie $\E_{\vtau \sim p_{\vtheta}} \left[ \| \mathbf{g}_{\vtheta}(\vtau) \|_2 \right]^2$. From Jensen's inequality, we immediately observe that:
\begin{align}
    \E_{\vtau \sim p_{\vtheta}} \left[ \| \mathbf{g}_{\vtheta}(\vtau) \|_2 \right]^2 \le \E_{\vtau \sim p_{\vtheta}} \left[ \| \mathbf{g}_{\vtheta}(\vtau) \|_2^2 \right],
\end{align}
and, consequently, we conclude that the off-policy gradient estimator with $ p_{*,\vtheta}$ as behavioral distribution suffers a smaller variance compared with the on-policy gradient estimator.

Furthermore, it is worth comparing the result of Theorem~\ref{thr:optimalQ} with the well-known result for minimum-variance estimation of expectation for non-negative scalar functions~\citep{kahn1950random}. Indeed, Theorem~\ref{thr:optimalQ} generalizes this result for vector-valued functions, reducing to the classical result for non-negative scalar functions, with the standard zero-variance estimator.

As already noted at the beginning of the section, although a convenient closed-form expression for the trajectory density function exists, it cannot be used in practice to collect trajectories since no policy exists inducing such a trajectory distribution. Nevertheless, it can be employed to learn a policy that induces a distribution as close as possible to this one.

\subsection{Cross-entropy minimization}\label{sec:XEntropy}
In this section, we illustrate how to employ the closed-form solution of the BPO problem derived in Section~\ref{sec:closedForm} in order to obtain a practical algorithm. Since, in practice, the parameter space $\vTheta$, together with the transition model, allows to span of a subset of the trajectory distributions $\Delta^{\vTau}$, we cannot represent the optimal behavioral distribution $p_*$ by means of a parametrization, \ie there not exists $\vtheta^b_* \in \vTheta$ such that $p_{*,\vtheta} = p_{\vtheta^b_*}$ a.s. However, we can conveniently project it into the space of representable behavioral distributions by minimizing the KL divergence:
\begin{align}\label{eq:KLProb}
    {\vtheta}^b_{\dagger} \in \argmin_{\vtheta^b \in \vTheta} D_{\text{KL}} \left( p_{*,\vtheta} \| p_{\vtheta^b} \right). 
\end{align}
This minimization problem can be further simplified into a weighted cross-entropy minimization by exploiting the functional form of $p_{*,\vtheta}$, as shown in the following result.

\begin{restatable}[]{prop}{propCS}
Let $p_{*,\vtheta}$ as defined in Equation~(\ref{eq:pStar}). Then, the solution to the problem in Equation~(\ref{eq:KLProb}) can be obtained via the weighted cross-entropy minimization:
\begin{align}
    \vtheta_{\dagger}^b \in \argmin_{\vtheta^b \in \vTheta} \E_{\vtau \sim p_{\vtheta}} \left[ -\|\mathbf{g}_{\vtheta} (\vtau)\| \log p_{\vtheta^b} (\vtau) \right] = \E_{\vtau \sim p_{\vtheta}} \left[ -\|\mathbf{g}_{\vtheta} (\vtau)\| \sum_{t=0}^{T-1}\log \pi_{\vtheta^b} (a_t|s_t) \right].
\end{align}
\end{restatable}

This alternative formulation has the advantage that the objective function is expressed as an expected value \wrt the trajectory distribution induced by the target policy, which can be estimated either on- or off-policy. In the most general case, we can resort to (multiple) off-policy estimation:
\begin{align} \label{eq:opt_behav_policy}
      \widehat{\vtheta}^b_{\dagger} \in \argmin_{\vtheta^b \in \vTheta} \frac{1}{n} \sum_{j=1}^m  \sum_{i=1}^{n_j}   \frac{p_{\vtheta}(\vtau_{ij})}{\bm{\Phi}_m(\vtau_{ij})}  \|\mathbf{g}_\vtheta(\vtau_{ij})\| \log p_{\vtheta^b}(\vtau_{ij}), \quad  \vtau_{ij} \sim p_{\vtheta^b_j}, \quad \forall i \in [n_j], \; \forall j \in [m].
\end{align}

\section{Theoretical Analysis}\label{sec:theory}
\begin{algorithm}[t]
\small
    \begin{algorithmic}[1]
        \State \textbf{Input:} initial target policy parameters $\vtheta_0$, batch sizes $N_{\mathrm{BPO}},N_{\mathrm{PG}}$, step size $\alpha$, defensive parameter $\beta$
        \For{$k=0,\dots,K-1$}
        \State $\mathcal{D}^{\mathrm{BPO}}_k = \left\{ N_{\mathrm{BPO}} \text{ trajectories collected with } {\vtheta}_{k}\right\}$
        \State $\widetilde{\vtheta}_k$ $\gets$ Solve (approximately) Equation~(\ref{eq:KLProb}) with $\mathcal{D}^{\mathrm{BPO}}_k$ \label{line:bpo}
        \State $\mathcal{D}^{\mathrm{PG}}_k = \left\{ (1-\beta)N_{\mathrm{PG}} \text{ trajectories } \vtau \sim p_{\widetilde{\vtheta}_{k}} \text{ and } \beta N_{\mathrm{PG}} \text{ trajectories } \vtau \sim p_{\vtheta_k} \right\}$
        \State $\bm{v}_k \gets \widehat{\nabla}J(\vtheta_k;\mathcal{D}^{\mathrm{PG}}_k)$  \label{line:off_grad} 
        \State $\vtheta_{k+1} \gets \vtheta_k + \alpha \bm{v}_k $
        \EndFor
        \State \textbf{return} $\vtheta_L$ with $L \sim \mathrm{Uni}([K])$
    \end{algorithmic}
    \caption{Policy Gradient with Behavioral Policy Optimization.}
    \label{alg:bpo_theory}
\end{algorithm}


In this section, we study the theoretical properties of Algorithm~\ref{alg:bpo_theory}, with a focus on the variance reduction granted by the active-IS estimator and how this impacts the rate of convergence of policy gradient to stationary points of the expected-return objective.

The quality of the policy gradient update will ultimately depend on how close our behavior policy is to the optimal one, and this cannot be ignored when deciding how many samples $N_{\mathrm{BPO}}$ are allocated to approximately solving Equation~(\ref{eq:KLProb}) in Line~\ref{line:bpo} of the algorithm. In Section~\ref{sec:theory_vr}, we first study the problem in full generality, assuming access to an $\epsilon$-minimizer of Equation~(\ref{eq:KLProb}). We remove this assumption in Section~\ref{sec:theory_rate}, studying the convergence rate for a specific but broad class of policies.

\subsection{Behavior Policy Optimization Oracle}\label{sec:theory_vr}
The following lemma shows the relationship between the variance of the off-policy estimator and the distance, in terms of chi-square divergence, between the chosen behavior distribution and the optimal one. It is given in terms of the variance reduction over Monte Carlo (on-policy) estimation.

\begin{restatable}[]{lemma}{varchi}\label{lem:var_chi2}
    Fix a target policy $\vtheta \in \vTheta$ and a behavior trajectory distribution $q \in \Delta^{\vTau}$. Let $\widehat{\nabla}_\vtheta J(\vtheta, \vtau)$ be the importance-weighted estimate of $\nabla_\vtheta J(\vtheta)$ computed with $\vtau \sim q$.
    Then the variance reduction from using $q$ in place of $p_\vtheta$ is given by:
    \begin{align*}
        \Var_{\vtau\sim p_{\vtheta}}\left[\widehat{\nabla}_\vtheta J(\vtheta;\vtau)\right] - \Var_{\vtau \sim q}\left[\widehat{\nabla}_\vtheta J(\vtheta; \vtau)\right] &= \Var_{\vtau\sim p_\vtheta}\left[\norm[2][]{\mathbf{g}_\vtheta(\vtau)}\right] - Z_\vtheta^2\chi^2(p_{*,\vtheta}\|q),
    \end{align*}
    where $Z_\vtheta \coloneqq \E_{\vtau\sim p_\vtheta}[\norm[2][]{\mathbf{g}_\vtheta(\vtau)}]$.
\end{restatable}
This lemma shows that the variance reduction depends on how closely we can approximate the optimal behavior distribution in terms of chi-square divergence. Unfortunately, the latter is hard to optimize from data. Using defensive samples reduces this to a KL-divergence error, which is much easier to control.
In this section, we just observe that the KL divergence can be made small using the approach proposed in Section~\ref{sec:XEntropy}, and operate under the following, more abstract:
\begin{assumption}[BPO Oracle]\label{asm:bpo}
For any target policy parameter $\vtheta\in\vTheta$, let $p_{*,\vtheta}$ be the corresponding optimal behavior distribution as defined in Equation~(\ref{eq:opt}). We assume access to a Behavioral Policy Optimization oracle $\mathrm{BPO}:\vTheta\to\vTheta$ that takes a target policy parameter $\vtheta$ and returns a behavior policy parameter $\widetilde{\vtheta}$ such that:
\begin{equation*}
    D_{\mathrm{KL}} \left( p_{*,\vtheta} \| p_{\widetilde{\vtheta}} \right) \le \epsilon_{\mathrm{KL}},
\end{equation*}
for some constant $\epsilon_{\mathrm{KL}}\ge 0$ independent of $\vtheta$.
\end{assumption}

The following theorem upper-bounds the excess variance in terms of the KL-divergence and provides a principled way to choose the defensive parameter $\beta$ in Algorithm~\ref{alg:bpo_theory}.

\begin{restatable}[]{theorem}{varkl}\label{th:var_kl}
    Fix a target policy $\vtheta \in \vTheta$ and a behavior policy $\widetilde{\vtheta}\in \vTheta$.
    Let $\beta \in [0,1]$ and let $\Phi=\beta p_\vtheta + (1-\beta)p_{\widetilde{\vtheta}}$ be the mixture trajectory distribution. Let $\widehat{\nabla}_\vtheta J(\vtheta;\vtau)$  be the $\beta$-defensive importance-weighted estimate of $\nabla_\vtheta J(\vtheta)$ computed with $\vtau \sim {\Phi}$. Then the variance reduction from using $\Phi$ in place of $p_\vtheta$ is at least
    \begin{align*}
         \Var_{\vtau\sim p_{\vtheta}}\left[\widehat{\nabla} J(\vtheta;\vtau)\right] - \Var_{\vtau\sim\Phi}\left[\widehat{\nabla}_\vtheta J(\vtheta;\vtau)\right] 
         &\ge \Var_{\vtau\sim p_\vtheta}\left[\norm[2][]{\mathbf{g}_\vtheta(\vtau)}\right] - 4Z_\vtheta(Z_\vtheta+\beta G_\vtheta)\left(2+\frac{1-\beta}{\beta}D_{\mathrm{KL}}(p_{*,\vtheta}\|p_{\widetilde{\vtheta}})\right),
    \end{align*}
    where $Z_\vtheta=\E_{\vtau\sim p_\vtheta}[\norm[2][]{\mathbf{g}_\vtheta(\vtau)}]$ and $G_\vtheta = \esssup_{\vtau\sim p_{\vtheta}}\{\enorm{\mathbf{g}_\vtheta(\vtau)}\}$. Under Assumption~\ref{asm:bpo}, provided $\epsilon_{\mathrm{KL}}\le 1$, by setting $\beta=\sqrt{\frac{\epsilon_{\mathrm{KL}}}{2-\epsilon_{\mathrm{KL}}}}$, the variance reduction is at least
    \begin{align}
        \Var_{\vtau\sim p_{\vtheta}}\left[\widehat{\nabla} J(\vtheta;\vtau)\right] - \Var_{\vtau\sim\Phi}\left[\widehat{\nabla}_\vtheta J(\vtheta;\vtau)\right] &\ge \Var_{\vtau\sim p_\vtheta}\left[\norm[2][]{\mathbf{g}_\vtheta(\vtau)}\right]
        -  4Z_\vtheta^2(2-\epsilon_{\mathrm{KL}}) - 4Z_\vtheta G_\vtheta \epsilon_{\mathrm{KL}} \nonumber \\&\qquad- 4Z_\vtheta(Z_\vtheta+G_\vtheta)\sqrt{\epsilon_{\mathrm{KL}}(2-\epsilon_{\mathrm{KL}})} \\
        &\ge \Var_{\vtau\sim p_\vtheta}\left[\norm[2][]{\mathbf{g}_\vtheta(\vtau)}\right]
        -  8Z_\vtheta^2 - 4Z_\vtheta(Z_\vtheta+2G_\vtheta)\sqrt{\epsilon_{\mathrm{KL}}}.
    \end{align}
\end{restatable}

\begin{remark}\label{rem:variance_reduction}
    As $\epsilon_\mathrm{KL}\to 0$, we have $\Var_{\vtau\sim p_{\vtheta}}[\widehat{\nabla} J(\vtheta;\vtau)] - \Var_{\widetilde{\vtau}\sim\Phi}[\widehat{\nabla} J(\vtheta;\vtau)] \ge \Var_{\vtau\sim p_\vtheta}[\norm[2][]{\mathbf{g}_\vtheta(\vtau)}]
        -  8Z_\vtheta^2 - o(\sqrt{\epsilon_\mathrm{KL}})$. Thus, if the KL-divergence is small enough, we there is variance reduction if
        \begin{equation}
            \Var_{\vtau\sim p_\vtheta}[\norm[2][]{\mathbf{g}_\vtheta(\vtau)}] = \E_{\vtau\sim p_\vtheta}[\enorm{\mathbf{g}_\vtheta(\vtau)}^2] - Z_\vtheta^2 > 9Z_\vtheta^2,
        \end{equation}
        that is, when $\E_{\vtau\sim p_\vtheta}[\enorm{\mathbf{g}_\vtheta(\vtau)}^2] > 10Z_\vtheta^2$. To see that variance reduction is indeed possible, consider the example: let $\vTau=\{\vtau_1,\vtau_2\}$ and the target distribution is $p_\theta$ such that $p_\theta(\vtau_1)=\theta$ and $p_\theta(\vtau_2)=1-\theta$, with $\theta\in[0,1]$. Suppose $g_\theta(\vtau_1)\in\{1,-1\}$ and $g_\theta(\vtau_2)=0$ for all $\theta$. Then $\E_{\vtau\sim p_\theta}[|g_\theta(\vtau)|^2]=\theta$, while $Z_\theta^2=\E_{\vtau\sim p_\theta}[|g_\theta(\vtau)|]^2=\theta^2$. So we can be sure there is variance reduction as long as $\theta<1/10$.
\end{remark}

We can use this result on variance reduction to upper bound the variance of the policy gradient estimates computed by our algorithm.
In the following, let $\mathcal{F}_k$ denote the sigma-algebra generated by all the random variables from Algorithm~\ref{alg:bpo_theory} up to iteration $k-1$ included, and all the trajectories from $\mathcal{D}^{\mathrm{BPO}}_{k}$. Note that both $\vtheta_k$ and $\widetilde{\vtheta}_k$ are $\mathcal{F}_k$-measurable. For brevity, we will write $\E_k[X]$ for the conditional expectation $\E[X|\mathcal{F}_k]$, and $\Var_k[X]$ for the conditional variance $\Var[X|\mathcal{F}_k]=\E_k[\enorm{X - \E_k[X]}^2]$ of a random element $X$. 

\begin{restatable}[]{theorem}{algVar}\label{th:alg_var}
    Fix an iteration $k\in[K]$ of Algorithm~\ref{alg:bpo_theory} and let $\mathcal{D}_{\mathrm{ON}}$ denote a dataset of $N_{\mathrm{PG}}$ independent trajectories collected with $\vtheta_k$. Under Assumption~\ref{asm:bpo}, the variance reduction granted by using the off-policy estimator $\mathbf{v}_k\coloneqq \widehat{\nabla} J(\vtheta_k;\mathcal{D}^{\mathrm{PG}}_k)$ with respect to an on-policy estimator is given by:
    \begin{equation}
        \Var_{k}\left[\widehat{\nabla} J(\vtheta_k;\mathcal{D}_{\mathrm{ON}})\right] - \Var_{k}[\mathbf{v}_k] \ge \frac{1}{N_\mathrm{PG}}\left(V_{k} - 8Z_{k}^2-4Z_k(Z_k+2G_k)\sqrt{\epsilon_{\mathrm{KL}}}\right),
    \end{equation}
    where $Z_{k}\coloneqq\E_{\vtau\sim p_{\vtheta_k}}[\enorm{\mathbf{g}_{\vtheta_k}(\vtau)}|\mathcal{F}_k]$, $V_{k}\coloneqq\Var_{\vtau\sim p_{\vtheta_k}}[\enorm{\mathbf{g}_{\vtheta_k}(\vtau)}|\mathcal{F}_k]$, and $G_k\coloneqq\esssup_{\vtau \sim p_{\vtheta_k}}\{\enorm{\mathbf{g}_{\vtheta_k}(\vtau)}\}$. Thus, the conditional variance of $\mathbf{v}_k$ is upper-bounded as follows:
    \begin{equation}
        \Var_k[\mathbf{v}_k] \le \frac{1}{N_{\mathrm{PG}}}\left(
            9Z_{k}^2 + Z_k(Z_k+2G_k)\sqrt{\epsilon_{\mathrm{KL}}} - \enorm{\nabla J(\vtheta_k)}^2
        \right).
    \end{equation}
\end{restatable}

\subsection{Convergence Rate}\label{sec:theory_rate}
So far, we studied the variance of the active-IS estimator from Algorithm~\ref{alg:bpo_theory}, showing that variance reduction is possible whenever the KL divergence between the optimal behavior distribution $p_{\vtheta,*}$ and its estimate $p_{\widetilde{\vtheta}}$ is small enough. We now give a more concrete characterization of the variance reduction in terms of how many on-policy samples are used to compute $p_{\widetilde{\vtheta}}$. We are only able to do so for a restricted class of policies, namely \emph{exponential-family} policies with linear sufficient statistics. However, this is a broad class that includes linear Gaussian and Softmax policies. Furthermore, this is the class of policies for which the (empirical) cross-entropy minimization problem described in Section~\ref{sec:XEntropy} admits a closed-form solution. Thus, it represents a setting where sample and computational efficiency can be achieved at the same time. Our analysis will also provide a principled way to allocate a per-iteration budget of $N$ trajectories in Algorithm~\ref{alg:bpo_theory}, that is, how to split them into $N_\mathrm{BPO}$ trajectories for behavior policy optimization, and $N_\mathrm{PG}$ trajectories for gradient estimation.  

We begin by listing all the assumptions that we will use in this section.

\begin{assumption}[Exponential-Family Policy]\label{asm:expfam}
    The target policy is of the form:
    \begin{equation*}
        \pi_\vtheta(a|s) = h(a)\exp\left(\vtheta^\top \bm{\bm{\varphi}}(s,a)-A(\vtheta,s)\right), \qquad \forall (s,a) \in \mathcal{S\times A},
    \end{equation*}
    where $\bm{\bm{\varphi}}:\mathcal{S}\times\mathcal{A}\to\Reals^d$ is the sufficient statistic, $h:\mathcal{A}\to\mathbb{R}_+$, and $A(\vtheta,s)=\log\int_{\mathcal{A}}h(a)\exp\left(\vtheta^\top\bm{\bm{\varphi}}(s,a)\right)\de a$ is the log-partition function.
\end{assumption}

This general model allows to conveniently represent widely used policies, including Gaussian policies with linear mean and Softmax policies~\citep{MetelliMR23}. 
Note that, for a policy satisfying Assumption~\ref{asm:expfam}, the score function is $
    \nabla_\vtheta \log \pi_\vtheta(a|s) = \bm{\varphi}(s,a) - \E_{a'\sim\pi_\vtheta(\cdot|s)}[\bm{\varphi}(s,a')] \eqqcolon \overline{\bm{\varphi}}_\vtheta(s,a)$,
and also
$
    \nabla_\vtheta^2\log\pi_\vtheta(a|s) = -\Cov_{a'\sim\pi_\vtheta(\cdot|s)}[\bm{\varphi}(s,a')]$.
We will refer to $\overline{\bm{\varphi}}_\vtheta$ as the \emph{centered} sufficient statistic. We now introduce a necessary assumption to guarantee that the optimal behavioral distribution over trajectories $p_{*,\vtheta^\dagger}$ is representable within the ones induced by the policies $\pi_{\vtheta^*}$ with $\vtheta^* \in \vTheta$.

\begin{assumption}[Realizability]\label{asm:real}
    For any target policy $\vtheta^\dagger\in\vTheta$, there exists $\vtheta^*\in\vTheta$ s.t.  the optimal behavior distribution w.r.t. $\vtheta^\dagger$ is $p_{*,\vtheta^\dagger}=p_{\vtheta^*}$, the trajectory distribution induced by policy $\pi_{\vtheta^*}$.
\end{assumption}

The next assumption is related to the tail behavior of the noise

\begin{assumption}[Subgaussianity]\label{asm:subgauss}
    For any $\vtheta\in\vTheta$ and $s\in\mathcal{S}$, the centered sufficient statistic $\overline{\bm{\varphi}}_\vtheta(s,\cdot)$ is $\sigma$-subgaussian in the sense that, for any $\bm{\lambda}\in\Reals^d$:
    \begin{equation*}
        \E_{a\sim\pi_\vtheta(\cdot|s)}\left[\exp\left(\bm{\lambda}^\top\overline{\bm{\varphi}}_\vtheta(s,a)\right)\right] \le \exp\bigg(\frac{\enorm{\bm{\lambda}}^2\sigma^2}{2}\bigg), \qquad \forall s \in \mathcal{S}.
    \end{equation*}
\end{assumption}

Finally, we enforce the following assumption that prescribes an exploration condition of the played policy encoded in a property of the spectrum of the empirical Fisher information matrix.

\begin{assumption}[Explorability]\label{asm:mineig}
    For a fixed target policy $\vtheta^\dagger \in \vTheta$ and a dataset of $n$ trajectories $\{\vtau_i\}_{i \in [n]}$ collected with $\pi_{\vtheta^\dagger}$  let
    \begin{equation}
        \widehat{\mathcal{F}}(\vtheta) = \frac{1}{n}\sum_{i=1}^n\enorm{\mathbf{g}_{\vtheta^\dagger}(\vtau_i)}\sum_{t=0}^{T-1}\Cov_{a\sim\pi_\vtheta(\cdot|s_t^i)}[\bm{\varphi}(s_t^i,a)].
    \end{equation}
    We assume that, for all $n\ge 1$ and $\vtheta^\dagger,\vtheta\in\vTheta$, $\E\left[\lambda_{\min}(\widehat{\mathcal{F}}(\vtheta))\right]\ge \lambda_* > 0$.
\end{assumption}


Given the previously listed assumptions, we are able to provide a meaningful bound on the expected error expressed in KL-divergence between the optimal behavioral trajectory distribution $p_{*,\vtheta}$ and the one estimated by the cross entropy minimization procedure $\widetilde{\vtheta}$.

\begin{restatable}[]{lemma}{mle}\label{lem:mle}
Fix a target policy parameter $\vtheta^\dagger \in \vTheta$ and let $\{\vtau_i\}_{i \in [n ]}$ be a dataset of $n$ i.i.d. trajectories collected with $\pi_{\vtheta^\dagger}$. Let 
\begin{align*}
    \widetilde{\vtheta}=\argmax_{\vtheta\in\vTheta} \sum_{i=1}^{n}\enorm{\mathbf{g}_{\vtheta^\dagger}(\vtau_i)}\sum_{t=0}^{T-1}\log\pi_{\vtheta}(a_t^i|s_t^i),
\end{align*}
and if $\esssup_{\vtau\sim p_{\vtheta}}\enorm{\mathbf{g}_{\vtheta}(\vtau)} \le G$ for all $\vtheta\in\vTheta$. Then, under Assumptions~\ref{asm:expfam},~\ref{asm:real},~\ref{asm:subgauss},~\ref{asm:mineig} it holds that:
    \begin{equation*}
        \E\left[D_\mathrm{KL}(p_{*,\vtheta^\dagger}\|p_{\widetilde{\vtheta}})\right] \le \frac{G^2T^3\sigma^4}{2\lambda_*^2n}.
    \end{equation*}
\end{restatable}

We are now ready to quantify the complete variance of the defensive off-policy estimator.

\begin{restatable}[]{theorem}{fullvar}\label{th:fullvar}
        Assuming $N_\mathrm{BPO}> \frac{G^2T^3\sigma^4}{2\lambda_*^2}$, let $\epsilon^*=\frac{G^2T^3\sigma^4}{2\lambda_*^2N_\mathrm{BPO}}$.
        Then, under Assumptions~\ref{asm:expfam},~\ref{asm:real},~\ref{asm:subgauss},~\ref{asm:mineig},
        Algorithm~\ref{alg:bpo_theory}
        with $\beta=\sqrt{\epsilon^*/(2-\epsilon^*)}$
        guarantees
        \begin{equation}
        \Var_k[\mathbf{v}_k] \le \frac{1}{N_{\mathrm{PG}}}\left(
            9Z_k^2 + \frac{Z_k(Z_k+2G)GT^{3/2}\sigma^2}{\lambda_*\sqrt{2N_\mathrm{BPO}}} - \enorm{\nabla J(\vtheta_k)}^2
        \right).
    \end{equation}
    Furthermore, by setting $N_\mathrm{BPO}=N_\mathrm{PG}=\frac{N}{2}$ and $\beta\in(0,1)$, provided $N>\frac{G^2T^3\sigma^4(1+\beta^2)}{2\lambda_*^2\beta^2}$ we have:
    \begin{equation}
        \Var_k[\mathbf{v}_k] \le \frac{1}{N}\left(18Z_k^2- \enorm{\nabla J(\vtheta_k)}^2\right) +\frac{Z_k(Z_k+2G)GT^{3/2}\sigma^2}{2\lambda_*N^{3/2}} .
    \end{equation}
\end{restatable}

We are finally able to provide the convergence rate of the corresponding iterative optimization.

\begin{restatable}[]{corollary}{rate}\label{cor:rate}
Let $\widetilde{V}\coloneqq 18Z_k^2- \enorm{\nabla J(\vtheta_k)}^2$ denote the residual variance left by the BPO process. Under the assumptions of Theorem~\ref{th:fullvar}, a total number of trajectories 
\begin{equation*}
    NK \le \left\lceil12(J(\vtheta^*)-J(\vtheta_0)) \left(\frac{3C_1\widetilde{V}}{\epsilon^{4}}+\frac{C_1+3C_2}{\epsilon^{10/3}}\right)\right\rceil
\end{equation*}
is sufficient for Algorithm~\ref{alg:bpo_theory} to obtain $\E[\enorm{\nabla J(\vtheta_{\text{out}})}] \le \epsilon$, where $\vtheta_{\text{out}}$ is chosen uniformly at random from the iterates $\vtheta_1,\dots,\vtheta_K$ of the algorithm, where $C_1=\frac{R_{\max}\sigma^2}{(1-\gamma)^2}$ and $C_2=\frac{R_{\max}^4\sigma^5\norm[\infty][]{\bm{\varphi}}(\sqrt{T}\sigma+2T\norm[\infty][]{\bm{\varphi}})T^3}{2\lambda_*(1-\gamma)^5}$.
\end{restatable}

\begin{remark}
    Although, in the worst case, the sample complexity is $O(\epsilon^{-4})$ like on-policy REINFORCE~\citep{yuan2022general}, when the residual variance $\widetilde{V}$ is negligible, namely, $\widetilde{V}=O(\epsilon^{2/3})$, Algorithm~\ref{alg:bpo_theory} can achieve an improved sample complexity of $O(\epsilon^{-10/3})$, the same as SVRPG~\citep{papini2018stochastic}. Examples of this can be constructed as in Remark~\ref{rem:variance_reduction}.
    Even though the optimal sample complexity for first-order policy optimization is $O(\epsilon^{-3})$ \citep{gu2020sample} and our $\epsilon^{2/3}$ improvement does not hold in full generality, we are not aware of any other case of provable acceleration of policy gradient algorithms following from behavior-policy optimization.
\end{remark}

\section{Related Works}

\textbf{Baselines}~~A common technique from statistical simulation to reduce variance in policy gradient estimation is using the \emph{baselines}. A baseline $\mathbf{b}$ is a non-random quantity that is subtracted from the return $R(\vtau)$ based on the observation that $\E_{\vtau \sim p_{\vtheta}}[\nabla \log p_{\vtheta}(\vtau) R(\vtau)] = \E_{\vtau \sim p_{\vtheta}}[\nabla \log p_{\vtheta}(\vtau)  (R(\vtau)- \mathbf{b})]$. Optimal baselines for the REINFORCE and G(PO)MDP estimators have been derived by~\cite{PetersS06}. Other approaches exploit a baseline that is obtained from a moving average of the most recent returns~\citep{weaver2001optimal, ZhaoHNS11}. This approach is similar to using a critic to estimate the value function~\citep{mei2022role}. The effectiveness of a baseline is highly problem-dependent and, in the end, does not change the convergence rate of the policy gradient algorithm, which remains of order $O(\epsilon^{-4})$, being $\epsilon$ the expected norm of the policy gradient reached.

\textbf{Variance-Reduced Policy Gradient Algorithms}~~
\emph{Variance reduction} techniques have been first introduced for supervised learning, having SVRG~\citep{johnson2013accelerating} as progenitor. The idea consists of re-using snapshots of gradients computed in the past to exploit the correlation in order to reduce the variance. Still, in the supervised learning community, several variations and improvements have been presented, which include SARAH~\citep{nguyen2017sarah}, STORM~\citep{cutkosky2019momentum} and PAGE~\citep{li2021page}. Each of these has been adapted to the policy gradient setting, giving rise to SVRPG~\citep{papini2018stochastic}, SRVR-PG~\citep{xu2020sample}, STORMPG~\citep{yuan2020stochastic}, and PAGEPG~\citep{gargiani2022page}, respectively. These approaches have succeeded in strictly improving the convergence rate over standard PGs. Indeed, SVRPG archives a convergence rate of order $O(\epsilon^{-10/3})$, as shown by~\cite{xu2020improved}, while SRVR-PG, STORMPG, and PAGEPG outperform it with a convergence rate of order $O(\epsilon^{-3})$, which is currently conjectured to be optimal.

\textbf{Active Importance Sampling}~~In~\cite{HannaTSN17}, the problem of \emph{behavioral policy search} is addressed with the goal of finding the most effective (i.e., minimum variance) behavioral policy to estimate the expected return of a given target policy. The approach is based on a gradient method that optimizes the policy parameters in order to find the minimum-variance behavioral policy. Although the approach demonstrated advantages from the policy evaluation perspective, it struggles to extend to policy optimization. In~\cite{hanna2019data}, the extension to the optimization perspective has been provided with a \emph{parallel policy search} approach that simultaneously optimizes over the parameters of the behavioral and target policies. Unfortunately, the algorithm enjoys no theoretical guarantees and shows limited empirical advantages. Recently, in~\cite{MetelliMR23}, the authors have deepened the connections between minimum-variance behavioral policy and the policy optimization have been studied. Specifically, under certain assumptions on the policy space, it is possible to show that the minimum variance behavioral policy attains a performance improvement. However, these works lack a comprehensive theoretical analysis capable of quantifying analytically the actual advantage of  \emph{active IS}, possibly in terms of convergence rate.

\section{Numerical Simulations}\label{sec:numerical}
In this section, we first provide a practical version of Algorithm~\ref{alg:bpo_theory} and then provide the experimental results on classical control tasks.

\subsection{Practical Algorithm} \label{sec:practical_algo}
Here, we present some practical aspects related to the implementation of Algorithm~\ref{alg:bpo_theory}, based on the above-introduced idea of IS estimators. In particular, in Algorithm~\ref{alg:bpo_theory}, we face two estimation problems: the estimation of KL divergence in Line~\ref{line:bpo} and the off-policy gradient estimation in Line~\ref{line:off_grad}. Both can benefit from effectively reusing already collected trajectories during the algorithm execution so as to reduce the overall number of samples generated per iteration.

\textbf{Offline KL divergence, Line~\ref{line:bpo}}~~In place of collecting, at every iteration $k$, new $N_{\mathrm{BPO}}$ trajectories with the current target policy $\pi_{\vtheta_k}$ to build the dataset $\mathcal{D}^{\mathrm{BPO}}_k$, we reuse the samples for the off-policy gradient estimation at the previous iteration $k-1$, namely $\mathcal{D}^{\mathrm{PG}}_{k-1}$. We call this KL estimation \emph{offline}, as it employs trajectories from the previous target and behavioral policies $\pi_{\vtheta_{k-1}}$ and $\pi_{\tilde{\vtheta}_{k-1}}$. Such offline samples need to be re-weighted proportionally to the probability of being generated by the current target policy $\pi_{\vtheta_k}$, for which we resort to the (multiple) off-policy estimator in Equation \ref{eq:opt_behav_policy}.

\textbf{Biased off-policy gradient, Line~\ref{line:off_grad}}~~The off-policy gradient estimation in Algorithm~\ref{alg:bpo_theory} is computed with the only behavioral policy $\pi_{\tilde{\vtheta}_k}$ and, when the defensive strategy is used $\beta>0$, with the current target policy $\pi_{\vtheta_k}$. To increase the number of trajectories available for the gradient estimation, we can reuse the already collected trajectories for the (offline) KL divergence estimation, namely $\mathcal{D}^{\mathrm{BPO}}_k$.
This approach is a multiple off-policy gradient estimator. If the offline KL strategy is employed, this means using the target policy $\pi_{\vtheta_{k-1}}$ at the previous iteration as an additional behavioral policy. Otherwise, $\mathcal{D}^{\mathrm{BPO}}_k$ contains \emph{biased} defensive samples from the current target policy $\pi_{\vtheta_k}$, as they were already used to compute the current behavioral policy $\pi_{\tilde{\vtheta}_k}$.

\subsection{Experimental Results} \label{sec:experimental_results}
All experiments are conducted with Gaussian policies with fixed diagonal variance, and the mean is linearly parametrized in the state so that $\pi_\vtheta = \mathcal{N}(\vtheta^\top s, \sigma \mathbf{I})$.
We first provide a set of numerical results on the Linear Quadratic Regulator (LQ) environment, quantifying the variance reduction of the single target policy gradient estimate; we then show the impact of such variance reduction on the learning iterations for solving the full control task in the Cartpole benchmark.
We employed the G(PO)MDP gradient estimator and its optimal baselines as derived in~\cite{PetersS06}.

\textbf{Variance Reduction}~~In this set of experiments, we want to analyze the impact of the optimal behavioral policy in estimating the target policy gradient. In particular, we compare the gradient variance (as defined in Equation (\ref{eq:opt})) in the on-policy and the proposed off-policy setting.
The optimal behavioral policy parameters $\widehat{\vtheta}^b_{\dagger}$ were computed by solving (\ref{eq:opt_behav_policy}), where the cross-entropy term was estimated by sampling $N_{\mathrm{BPO}}$ trajectories from the target policy $\pi_{\vtheta}$. Afterwards, $N_{\mathrm{PG}}$ trajectories were sampled from the behavioral $\pi_{\widehat{\vtheta}^b_{\dagger}}$ to build the data-set $\mathcal{D}_{\text{off}}$ and compute the off-policy gradient as in equation~(\ref{eq:off_grad}). The on-policy gradient estimations were instead obtained with a batch of $N_{\mathrm{BPO}} + N_{\mathrm{PG}}$ trajectories forming the data-set $\mathcal{D}_{\text{on}}$.

We run exhaustive experiments by varying the LQ horizon and the state dimensions. The complete results are reported in Appendix~\ref{sec:appendix_numerical_results}. Here, we fix the horizon to 2 and consider mono-dimensional LQ problems varying parameters of the target policy, i.e., various $\vtheta \in \{ -1.0, -0.5, 0.0, 0.5, 1.0\}$ and log standard deviations $\log \sigma \in \{-1.0, -0.5, 0.0, 0.5, 1.0\}$. Finally, we varied also the hyper-parameters of our off-policy method, i.e. the defensive coefficient $\beta \in \{0, 0.4, 0.8\}$, the biased off-policy practical gradient calculation (the offline estimation of the KL divergence here is not possible), and the batch sizes $N_{\mathrm{BPO}}$ (10, 30 and 50) and $N_{\mathrm{PG}} \in \{90, 70, 50\}$.
Tables \ref{tab:var_lq_mu} and \ref{tab:var_lq_std} report, for each environment and policy configuration, the first 20 results  ordered by the average variance gap between the on-policy and off-policy methods (over 100 repetitions), i.e.:
\begin{equation}
     {\Delta \! \Var} = \frac{1}{100} \sum_{i=1}^{100} \left(
     \Var \left[\widehat{\nabla} J(\vtheta; \mathcal{D}_{\text{on}}^{(i)}) \right]
     -  \Var \left[\widehat{\nabla} J(\vtheta; \mathcal{D}_{\text{off}}^{(i)}) \right] \right).
\end{equation}
Across all the results, we can notice a few prevalent patterns. Firstly, as may be expected, the variance reduction is numerically more significant for "extreme" values of the policy parameters ($\vtheta$ and $\log \sigma$ close to 1), as the gradient estimation problem becomes more and more difficult and prone to high variance, thus leading to significant margin of improvement (see also the complete results in Appendix~\ref{sec:appendix_numerical_results}).
Secondly, the biased off-policy gradient calculation is predominant in most of the highest variance reduction results, as it allows the use of the same number of samples of the on-policy counterpart.
Lastly, the other off-policy hyper-parameters do not seem to impact these variance reduction results clearly, alternating different combinations in the best experiments reported in all the tables.

\begin{table}
\caption{LQ environment, with horizon = 2 and state dimension = 1. Variance reduction in off-policy gradient, expressed as $\Delta \! \Var$ and its 95\% Gaussian confidence interval $(\Delta \! \Var^-,\Delta \! \Var^+)$, with different hyper-parameters.}
\sisetup{ round-mode = places, round-precision = 2 }
\small
\setlength{\tabcolsep}{3pt}
\begin{subtable}{.5\textwidth}
\centering
\caption{Target policy with $\log \sigma = 0$ and varying $\vtheta$.}\label{tab:var_lq_mu}
\resizebox{.97\textwidth}{!}{\begin{tabular}{S S S c >$c<$ >$c<$ >$c<$ S[round-precision=1]}
\toprule
\text{$\Delta \! \Var$} & \text{$\Delta \! \Var^-$} & \text{$\Delta \! \Var^+$} & biased & \beta & N_{\mathrm{BPO}} & N_{\mathrm{PG}} & $\vtheta$\\
\midrule
2.048126 & 1.127738  & 2.968514 & True  & 0.8 & 50 & 50 & 1.0 \\
1.643050 & -0.103186 & 3.389286 & True  & 0.0 & 10 & 90 & 1.0 \\
1.503771 & 0.775425  & 2.232117 & True  & 0.4 & 50 & 50 & 1.0 \\
1.388987 & 0.323475  & 2.454499 & False & 0.0 & 10 & 90 & 1.0 \\
1.260688 & 0.628243  & 1.893133 & True  & 0.0 & 50 & 50 & 1.0 \\
1.148023 & -0.616740 & 2.912785 & True  & 0.8 & 10 & 90 & 1.0 \\
0.703738 & 0.253894  & 1.153583 & True  & 0.0 & 30 & 70 & -1.0 \\
0.562928 & -0.714201 & 1.840058 & False & 0.0 & 50 & 50 & 1.0 \\
0.561355 & 0.302557  & 0.820153 & True  & 0.8 & 50 & 50 & -1.0 \\
0.508645 & 0.032941  & 0.984350 & True  & 0.0 & 10 & 90 & -1.0 \\
0.469772 & 0.258537  & 0.681006 & True  & 0.4 & 50 & 50 & -1.0 \\
0.408988 & 0.144608  & 0.673367 & True  & 0.4 & 50 & 50 & 0.5 \\
0.398966 & 0.183075  & 0.614856 & True  & 0.0 & 50 & 50 & -1.0 \\
0.392953 & -0.018980 & 0.804886 & False & 0.0 & 10 & 90 & 0.5 \\
0.324804 & 0.159038  & 0.490571 & True  & 0.0 & 30 & 70 & 0.5 \\
0.321055 & -0.169663 & 0.811773 & False & 0.4 & 10 & 90 & -1.0 \\
0.311033 & -0.068349 & 0.690415 & False & 0.0 & 10 & 90 & -1.0 \\
0.306358 & -0.159384 & 0.772100 & False & 0.4 & 50 & 50 & -1.0 \\
0.296821 & 0.071193  & 0.522449 & True  & 0.8 & 50 & 50 & 0.5 \\
0.290852 & -0.136284 & 0.717988 & False & 0.8 & 10 & 90 & -1.0 \\
0.270046 & -0.144759 & 0.684851 & True  & 0.4 & 10 & 90 & -1.0 \\
\bottomrule
\end{tabular}}
\end{subtable}%
\sisetup{ round-mode = places, round-precision = 2 }
\small
\setlength{\tabcolsep}{3pt}
\begin{subtable}{.5\textwidth}
\centering
\caption{Target policy with $\vtheta = 0$ and varying $\log \sigma$.}\label{tab:var_lq_std}
\resizebox{.97\textwidth}{!}{\begin{tabular}{S S S c >$c<$ >$c<$ >$c<$ S[round-precision=1]}
\toprule
\text{$\Delta \! \Var$} & \text{$\Delta \! \Var^-$} & \text{$\Delta \! \Var^+$} & biased & \beta & N_{\mathrm{BPO}} & N_{\mathrm{PG}}  & $\log \sigma$\\
\midrule
4.041229 & 2.016358  & 6.066101 & True  & 0.8 & 50 & 50 & 1.0 \\
3.773828 & 2.399339  & 5.148317 & True  & 0.4 & 50 & 50 & 1.0 \\
3.247087 & 1.631413  & 4.862761 & True  & 0.0 & 30 & 70 & 1.0 \\
3.176471 & 1.951825  & 4.401116 & True  & 0.0 & 50 & 50 & 1.0 \\
2.700678 & 0.718395  & 4.682962 & True  & 0.8 & 30 & 70 & 1.0 \\
2.361232 & -0.389329 & 5.111792 & True  & 0.4 & 30 & 70 & 1.0 \\
2.055510 & 0.523027  & 3.587994 & True  & 0.0 & 10 & 90 & 1.0 \\
1.535046 & -0.378748 & 3.448839 & False & 0.0 & 10 & 90 & 1.0 \\
1.186207 & -0.690749 & 3.063163 & False & 0.0 & 30 & 70 & 1.0 \\
0.604685 & -1.277262 & 2.486632 & False & 0.8 & 30 & 70 & 1.0 \\
0.590495 & 0.244142  & 0.936848 & True  & 0.8 & 50 & 50 & 0.5 \\
0.581094 & -1.889402 & 3.051590 & False & 0.0 & 50 & 50 & 1.0 \\
0.556353 & 0.215147  & 0.897560 & True  & 0.0 & 50 & 50 & 0.5 \\
0.483820 & 0.186378  & 0.781262 & True  & 0.4 & 50 & 50 & 0.5 \\
0.419434 & 0.145579  & 0.693288 & True  & 0.8 & 30 & 70 & 0.5 \\
0.392720 & -1.539439 & 2.324879 & False & 0.4 & 30 & 70 & 1.0 \\
0.240989 & -0.039873 & 0.521851 & True  & 0.8 & 10 & 90 & 0.5 \\
0.227993 & -0.026717 & 0.482703 & True  & 0.4 & 30 & 70 & 0.5 \\
0.161052 & -0.278498 & 0.600601 & True  & 0.0 & 10 & 90 & 0.5 \\
0.157868 & -0.245560 & 0.561297 & False & 0.0 & 50 & 50 & 0.5 \\
0.151897 & -0.199452 & 0.503245 & False & 0.0 & 30 & 70 & 0.5 \\
\bottomrule
\end{tabular}}
\end{subtable}
\end{table}
%
%

\textbf{Learning Speed-up}~~In this second set of experiments, we want to measure the impact of the variance reduction provided by our off-policy method in the learning process for solving the classic Cartpole balancing problem and compare our results with the state-of-art variance reduction algorithm STORMPG.
For our off-policy algorithm, we chose $\beta=0$, and employed both the practical aspects of the offline KL divergence estimation (hence we do not use $N_{\mathrm{BPO}}$) and of biased off-policy gradient estimation (see Section~\ref{sec:practical_algo}). All the experiments were run with a fixed budget of $N_{\mathrm{PG}}$ samples for each iteration, which also correspond to the mini batch-size employed by the STORMPG (the initial batch-size was set to 10 times $N_{\mathrm{PG}}$).
Figure~\ref{fig:learning_cartpole} shows that our off-policy method outperforms the STOMRPG in all the different configurations, enjoying both a more stable behavior at convergence and a lower variance during the learning iterations.

\begin{figure}[h!t]
\centering
\begin{subfigure}{.32\textwidth}
    \centering
    \includegraphics[width=\textwidth]{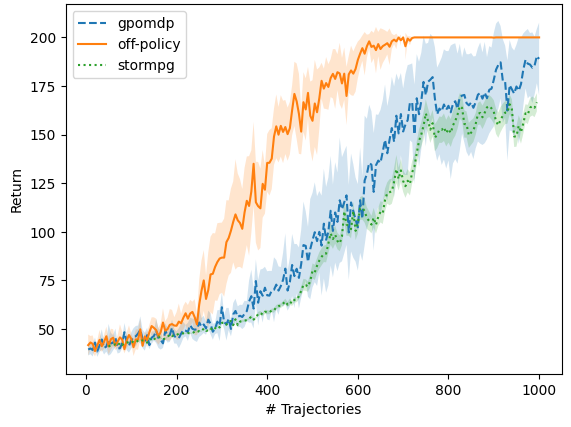}
    \caption{}
\end{subfigure}
\begin{subfigure}{.32\textwidth}
    \centering
    \includegraphics[width=\textwidth]{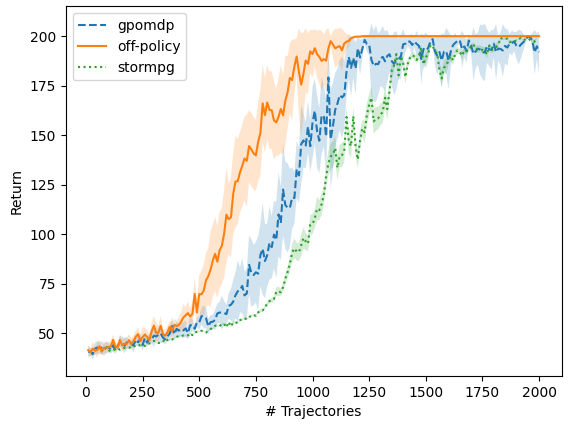}
    \caption{}
\end{subfigure}
\begin{subfigure}{.32\textwidth}
    \centering
    \includegraphics[width=\textwidth]{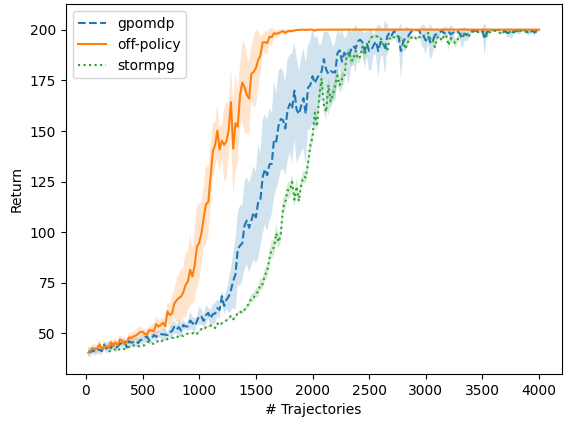}
    \caption{}
\end{subfigure}
\begin{subfigure}{.32\textwidth}
    \centering
    \includegraphics[width=\textwidth]{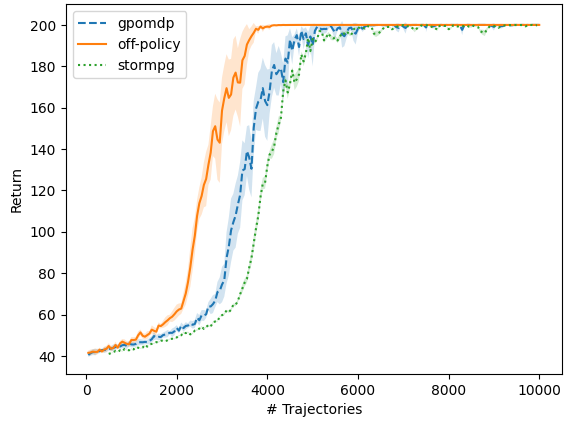}
    \caption{}
\end{subfigure}
\begin{subfigure}{.32\textwidth}
    \centering
    \includegraphics[width=\textwidth]{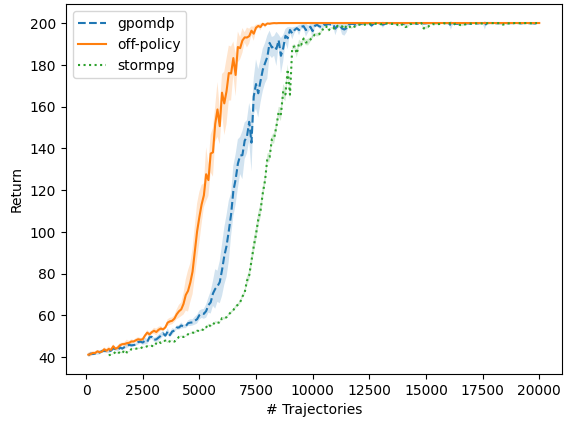}
    \caption{}
\end{subfigure}
\caption{Cartpole. Average return and its 95\% Gaussian CI (30 repetitions) over the learning iterations. Different policy gradient batch-sizes were used: (a) $N_{\mathrm{PG}} = 5$, (b) $N_{\mathrm{PG}} = 10$, (c) $N_{\mathrm{PG}} = 20$, (d) $N_{\mathrm{PG}} = 50$, (e) $N_{\mathrm{PG}} = 100$.}
\label{fig:learning_cartpole}
\end{figure}

\section{Discussion and Conclusions}
In this paper, we have presented a novel approach to control the variance of the PG estimator. Leveraging the idea of looking for the best behavioral policy that minimizes the variance of the IS estimator, we have introduced a novel algorithm that exploits a two-phase procedure, alternating between the cross-entropy estimation of such a policy and the actual off-policy performance improvement. We have shown that, thanks to the defensive estimate, we are able to achieve a convergence rate of order $O(\epsilon^{-4})$ to a stationary point. Compared to the standard REINFORCE convergence rate, our algorithm enjoys a smaller residual variance. Then, we provided a practical version of such an algorithm, which uses all the samples collected so far at the price of an estimation bias. This algorithm was evaluated on benchmark continuous control tasks compared to standard baselines, showing a significant reduction of the estimation variance and a faster learning curve. Future works include the extension of the provided algorithm in combination with variance reduction techniques, such as SVRPG, and the conception of a more practical adaptation that suitably combines with deep architectures.

\bibliography{main}
\bibliographystyle{rlc}

\clearpage

\appendix




\section{Hellinger Distance}
The  Hellinger distance between two distributions $P\ll Q$ is defined as\footnote{In some texts, the Hellinger distance is normalized by $\sqrt{2}$ to be in $[0,1]$.}
\begin{equation}
    D_H(P,Q) = \sqrt{\int_{\vTau} \left(\sqrt{p(\vtau)}-\sqrt{q(\vtau)}\right)^2\de\vtau}.
\end{equation}

In the following we list some known properties of the Hellinger distance that will be useful in our proofs. See, for instance,~\citep{foster2021efficient}.
\begin{itemize} 
    \item Boundedness: $D_H(P,Q)\le \sqrt{2}$.
    \item The Hellinger distance is a metric. In particular, we will use symmetry, $D_H(P,Q)=D_H(Q,P$), and the fact that $D_H(P,P) = 0$.
    \item The squared Hellinger distance is an f-divergence. In particular, we will use the joint convexity of f-divergences: $D_H^2(\beta P_1+(1-\beta)P_2,\beta Q_1 + (1-\beta)Q_2) \le \beta D_H^2(P_1,Q_1) + (1-\beta)D_H^2(P_2,Q_2)$. By taking $P_2=Q_1=Q_2$, we have $D_H^2(P, \beta P +(1-\beta)Q) \le (1-\beta)D_H(P,Q)$.
    \item Pinsker-style inequality: $D_H(P,Q) \le \sqrt{\min\{D_{\mathrm{KL}}(P\|Q),D_{\mathrm{KL}}(Q\|P)\}}$.
\end{itemize}

\section{Omitted Proofs}\label{app:proofs}

\subsection{Proofs of Section~\ref{sec:bpo}}

\optimalQ*
\begin{proof}
    We consider a probability measure over the trajectory space $p \in \Delta^{\vTau}$. Let us first observe that since the off-policy estimator is unbiased, we can focus on the second moment:
    \begin{align}
        \Var_{\vtau \sim p_{\vtheta^b}}\left[\widehat{\nabla} J(\vtheta;\vtau)\right]  & =   \E_{\vtau \sim p_{\vtheta^b}} \left[ \left\|  \frac{p_{\vtheta}(\vtau)}{p_{\vtheta^b}(\vtau)} \mathbf{g}_{\vtheta}(\vtau) - \nabla J(\vtheta) \right\|^2_2 \right] \\
        & = \E_{\vtau \sim p_{\vtheta}} \left[ \left(  \frac{p_{\vtheta}(\vtau)}{p_{\vtheta^b}(\vtau)}\right)^2 \left\| \mathbf{g}_{\vtheta}(\vtau) \right\|^2_2 \right] - \left\|\nabla J(\vtheta) \right\|^2_2 
    \end{align}
    where the first inequality follows from the independence of the trajectories. Thus, we consider the following optimization problem, where the expectations are written with the corresponding integrals for convenience:
    \begin{align}
        & \min_{p \in \Delta^{\vTau}}  \int_{\vTau}  \frac{p_{\vtheta}(\vtau)^2}{p(\vtau)} \left\| \mathbf{g}_{\vtheta}(\vtau) \right\|^2_2 \de \vtau \\
        & \text{s.t.} \quad \int_{\vTau} p(\vtau) \de \vtau = 1\label{eq:c2}\\
        & \phantom{\text{s.t.}} \quad p(\vtau) \ge 0 \qquad \forall \vtau \in \vTau \label{eq:c3}
    \end{align}
    The problem has a convex objective function and linear constraints. Thus, we approach it with the Lagrange multipliers, dropping the non-negativity constraint that, as we shall see, will be already ensured by the derived solution. Let $\lambda \in \Reals$:
    \begin{align}
        L(p(\cdot),\lambda) =  \int_{\vTau}  \frac{p_{\vtheta}(\vtau)^2}{p(\vtau)} \left\| \mathbf{g}_{\vtheta}(\vtau) \right\|^2_2 \de \vtau  + \lambda \left( \int_{\vTau} p(\vtau) \de \vtau - 1 \right). 
    \end{align}
    By vanishing the functional derivative \wrt $p(\cdot)$, we obtain for every $\vtau \in \vTau$:
    \begin{align}
        \frac{\delta L(p(\cdot),\lambda)}{ \delta p(\cdot)} (\vtau) =- \frac{p_{\vtheta}(\vtau)^2}{p(\vtau)^2} \left\| \mathbf{g}_{\vtheta}(\vtau) \right\|^2_2 + \lambda = 0 \implies p(\vtau) =\sqrt{\lambda}  p_{\vtheta}(\vtau) \|\mathbf{g}_{\vtheta}(\vtau)\|_2,
    \end{align}
    having retained the non-negative solution only.
    Since for constraint~\ref{eq:c3}, the density must integrate up to $1$, we have that for every $\vtau \in \vTau$:
    \begin{align}
        p(\vtau) = \frac{p_{\vtheta}(\vtau) \| \mathbf{g}_{\vtheta}(\vtau) \|_2}{\int_{\bm{\mathcal{T}}} p_{\vtheta}(\vtau) \| \mathbf{g}_{\vtheta}(\vtau) \|_2 \de \vtau}.
    \end{align}
\end{proof}

\propCS*
\begin{proof}
    We simply exploit the form of the optimal behavioral distribution $p_*$ and the definition of KL divergence:
    \begin{align}
        \argmin_{\vtheta^b \in \vTheta} D_{\text{KL}} \left( p_{*,\vtheta} \| p_{\vtheta^b} \right) & = \argmin_{\vtheta^b \in \vTheta} \E_{\vtau \sim p_{*,\vtheta}} \left[ \log \left( \frac{p_{*,\vtheta}(\vtau)}{p_{\vtheta^b}(\vtau)} \right) \right] \\
        & = \argmin_{\vtheta^b \in \vTheta} - \E_{\vtau \sim p_{*,\vtheta}} \left[ \log p_{\vtheta^b}(\vtau)  \right] \\
        & = \argmin_{\vtheta^b \in \vTheta} - \int_{\vTau} \frac{p_{\vtheta}(\vtau) \| \mathbf{g}_{\vtheta}(\vtau)\|_2}{p_{\vtheta}(\vtau') \| \mathbf{g}_{\vtheta}(\vtau')\|_2 \de \vtau'} \log p_{\vtheta^b}(\vtau)  \de \vtau \\
        & = \argmin_{\vtheta^b \in \vTheta} -\E_{\vtheta \sim p_{\vtheta}} \left[ \|\mathbf{g}_{\vtheta} (\vtau)\| \log p_{\vtheta^b} (\vtau) \right] ,
    \end{align}
    which proves the first equality. For the second equality, we observe that:
    \begin{align}
        \log p_{\vtheta^b} (\vtau) = \log \mu_0(s_0) + \sum_{t=0}^{T-1} \log \pi_{\vtheta^b}(a_t|s_t) + \sum_{t=0}^{T-1} \log P(s_{t+1}|s_t,a_t),
    \end{align}
    and that the addenda of the initial-state distribution and of the transition model do not depend on $\vtheta^b$.
\end{proof}

\subsection{Proofs of Section~\ref{sec:theory_vr}}

\varchi*
\begin{proof}

Let $p_*$ be short for $p_{*,\vtheta}$.
    First, we know from Theorem~\ref{thr:optimalQ} that the variance reduction granted by the optimal behavior distribution w.r.t. on-policy estimation is
    \begin{align*}
         \Var_{\vtau\sim p_{\vtheta}}[\widehat{\nabla} J(\vtheta;\vtau)]- \Var_{\vtau\sim p_*}[\widehat{\nabla}J(\vtheta; \vtau)] = \E_{\vtau\sim p_\vtheta}[\enorm{g_\vtheta(\vtau)}^2]- \E_{\vtau\sim p_\vtheta}[\enorm{g_\vtheta(\vtau)}]^2 = \Var_{\vtau\sim p_\vtheta}[\enorm{g_\vtheta(\vtau)}].
    \end{align*}
    Let $\bm{v} = $, so the variance reduction granted by sampling from $q$ is
    \begin{align}
        \Var_{\vtau\sim p_{\vtheta}}[\widehat{\nabla} J(\vtheta;\vtau)] - \Var_{\vtau\sim q}[\widehat{\nabla} J(\vtheta;\vtau)] &= 
        \Var_{\vtau\sim p_{\vtheta}}[\widehat{\nabla} J(\vtheta;\vtau)] - \Var_{\vtau\sim p_*}[\widehat{\nabla}J(\vtheta;\vtau)]
        \nonumber\\&\qquad+
        \Var_{\vtau\sim p_*}[\widehat{\nabla}J(\vtheta; p_*;\vtau)]
        - \Var_{\vtau\sim q}[\widehat{\nabla} J(\vtheta;\vtau)] \\
        &=\Var_{\vtau\sim p_\vtheta}[\enorm{g_\vtheta(\vtau)}] - \left(\Var_{\vtau\sim q}[\widehat{\nabla} J(\vtheta;\vtau)] -  \Var_{\vtau\sim p_*}[\widehat{\nabla}J(\vtheta; p_*;\vtau)]\right),
    \end{align}
    which is the variance reduction granted by $p_*$ minus the excess variance due to using a proxy $q$ of $p_*$. We can characterize this excess variance as follows.
    Since both estimates are unbiased:
    \begin{align}
        \Var_{\vtau\sim q}[\widehat{\nabla}J(\vtheta;\vtau)] &- \Var_{\vtau\sim p_*}[\widehat{\nabla}J(\vtheta;\vtau)] 
        = \E_{\vtau\sim q}\left[\enorm{\widehat{\nabla}J(\vtheta;\vtau)}^2\right] - \E_{\vtau\sim p_*}\left[\enorm{\widehat{\nabla}J(\vtheta;\vtau)}^2\right] \\
        &=\int_{\vTau}q(\vtau)\frac{p_\vtheta(\vtau)^2}{q(\vtau)^2}\enorm{g_\vtheta(\vtau)}^2 \de\vtau - \int_{\vTau}p_*(\vtau)\frac{p_\vtheta(\vtau)^2}{p_*(\vtau)^2}\enorm{g_\vtheta(\vtau)}^2 \de\vtau \\
        &=\int_{\vTau}p_\vtheta(\vtau)\frac{p_\vtheta(\vtau)}{q(\vtau)}\enorm{g_\vtheta(\vtau)}^2 \de\vtau - \int_{\vTau}p_\vtheta(\vtau)\frac{p_\vtheta(\vtau)}{p_*(\vtau)}\enorm{g_\vtheta(\vtau)}^2 \de\vtau \\
        &= Z_\vtheta\int_{\vTau}p_*(\vtau)\frac{p_\vtheta(\vtau)}{q(\vtau)}\enorm{g_\vtheta(\vtau)} \de\vtau - Z_\vtheta \int_{\vTau}p_\vtheta(\vtau)\enorm{g_\vtheta(\vtau)} \de\vtau \label{pp:1}\\
        &= Z_\vtheta \int_{\vTau}\frac{p_\vtheta(\vtau)}{q(\vtau)}\enorm{g_\vtheta(\vtau)}\left(p_*(\vtau)-q(\vtau)\right) \de\vtau\\
        &= Z_\vtheta^2 \int_{\vTau}\frac{p_*(\vtau)}{q(\vtau)}\left(p_*(\vtau)-q(\vtau)\right) \de\vtau \label{pp:2} \\
        &= Z_\vtheta^2 \left(\int_{\vTau}\frac{p_*(\vtau)^2}{q(\vtau)}\de\vtau - 1\right) \\
        &= Z_\vtheta^2\chi^2(p_*\|q),
    \end{align}
    where Equation~(\ref{pp:1}) and (\ref{pp:2}) are by definition of $p_*$.
\end{proof}


Unfortunately, it is not possible to upper bound the chi-square divergence in terms of the KL in general. To obtain an upper bound for the special case of defensive estimators, we will need the following technical lemma, a generalization of Lemma 8 by~\cite{bubeck2020first}.
\begin{lemma}\label{lem:chiplus}
   For any $\eta>0$,
    \begin{equation*}
        \int_{\mathcal{\vTau}}\frac{(q(\vtau)-p(\vtau))^2}{q(\vtau)}\mathbf{1}_{\{q(\vtau)\ge \eta p(\vtau)\}} \de\vtau  \le 4\eta^{-3/2} D_H^2(p,q).
    \end{equation*}
\end{lemma}
\begin{proof}
    Let $f_t(s)=(\sqrt{t}-\sqrt{s})^2$. Its second derivative is $f_t''(s)=\frac{\sqrt{t}}{2s\sqrt{s}}$. We can see that, restricted to $s\le\eta^{-1} t$, $f_t$ is $\frac{\eta^{3/2}}{2t}$-strongly convex. Hence:
    \begin{equation}
        f_t(s) \ge \frac{\eta^{3/2}(t-s)^2}{4t}.
    \end{equation}
    Letting $t=q(\vtau)$ and $s=p(\vtau)$ and using the definition of Hellinger distance:
    \begin{align}
        D_H^2(p,q) = \int_{\vTau} \left(\sqrt{q(\vtau)}-\sqrt{p(\vtau)}\right)^2 \de\vtau 
        &\ge \int_{\vTau} \left(\sqrt{q(\vtau)}-\sqrt{p(\vtau)}\right)^2 \mathbf{1}_{\{q(\vtau)\ge \eta p(\vtau)\}}\de\vtau \\
        &\ge \frac{\eta^{3/2}}{4}\int_{\vTau} \frac{(q(\vtau)-p(\vtau))^2}{q(\vtau)} \mathbf{1}_{\{q(\vtau)\ge \eta p(\vtau)\}}\de\vtau.
    \end{align}
\end{proof} 


We are now ready to prove Theorem~\ref{th:var_kl}.

\varkl*
\begin{proof}
    To prove the first lower bound on variance reduction, we use Lemma~\ref{lem:var_chi2} with $\phi$ (density of $\Phi$) in place of $q$ and upper bound the negative term as follows, applying Lemma~\ref{lem:chiplus} twice:
    \begin{align}
        Z_\vtheta^2\chi^2(p_*\|\Phi) 
        &= Z_\vtheta^2 \int_{\vTau} \frac{(\phi(\vtau)-p_*(\vtau))^2}{\phi(\vtau)} \mathbf{1}_{\{\phi(\vtau)\ge \beta^{2/3}p_*(\vtau)\}}\de \vtau  + Z_\vtheta^2 \int_{\vTau} \frac{(\phi(\vtau)-p_*(\vtau))^2}{\phi(\vtau)} \mathbf{1}_{\{\phi(\vtau)\le \beta^{2/3}p_*(\vtau)\}}\de \vtau \\
        &\le Z_\vtheta^2 \frac{4}{\beta}D_H^2(p_*,\phi)  + Z_\vtheta^2 \int_{\vTau} \frac{(\phi(\vtau)-p_*(\vtau))^2}{\phi(\vtau)} \mathbf{1}_{\{\phi(\vtau)\le \beta^{2/3}p_*(\vtau)\}}\de \vtau \label{pp:4}\\
        &\le \frac{4Z_\vtheta^2}{\beta}D_H^2(p_*,\phi)  + \frac{Z_\vtheta^2}{\beta} \int_{\vTau} \frac{(\phi(\vtau)-p_*(\vtau))^2}{p(\vtau)} \mathbf{1}_{\{\phi(\vtau)\le \beta^{2/3}p_*(\vtau)\}}\de \vtau \\
        &= \frac{4Z_\vtheta^2}{\beta}D_H^2(p_*,\phi)  + \frac{Z_\vtheta}{\beta} \int_{\vTau}\enorm{g_\vtheta(\vtau)} \frac{(\phi(\vtau)-p_*(\vtau))^2}{p_*(\vtau)} \mathbf{1}_{\{p_*(\vtau)\ge \beta^{-2/3}\phi(\vtau)\}}\de \vtau \\
        &\le \frac{4Z_\vtheta^2}{\beta}D_H^2(p_*,\phi)  + \frac{Z_\vtheta G_\vtheta}{\beta} \int_{\vTau}\frac{(\phi(\vtau)-p_*(\vtau))^2}{p_*(\vtau)} \mathbf{1}_{\{p_*(\vtau)\ge \beta^{-2/3}\phi(\vtau)\}}\de \vtau \\
        &\le \frac{4Z_\vtheta^2}{\beta}D_H^2(p_*,\phi)  + 4Z_\vtheta G_\vtheta D_H^2(\phi,p_*) \label{pp:3}\\
        &\le 4Z_{\vtheta}\frac{Z_\vtheta+\beta G_\vtheta}{\beta}\left(\beta D_H^2(p_*,p)+(1-\beta)D_H^2(p_*,p_{\widetilde{\vtheta}})\right) \\
        &\le 4Z_{\vtheta}\frac{Z_\vtheta+\beta G_\vtheta}{\beta}\left(2\beta+(1-\beta)D_{\mathrm{KL}}(p_*\|p_{\widetilde{\vtheta}})\right) \\
        &= 4Z_\vtheta(Z_\vtheta+\beta G)\left(2+\frac{1-\beta}{\beta}D_{\mathrm{KL}}(p_*\|p_{\widetilde{\vtheta}})\right),
    \end{align}
    where the inequalities~(\ref{pp:4}) and (\ref{pp:3}) are by Lemma~\ref{lem:chiplus}.
    The latter expression is convex in $\beta$, but the optimal value $\beta^*=\sqrt{\frac{Z_\vtheta\epsilon_{\mathrm{KL}}}{(2-\epsilon_{\mathrm{KL}})G_\vtheta}}$ cannot be computed since $Z_\vtheta$ is unknown. However, upper-bounding $Z_\vtheta$ by $G_\vtheta$ and setting\footnote{Note that we do not actually need to know $G_\vtheta$, nor an upper bound.} $\beta=\sqrt{\frac{\epsilon_{\mathrm{KL}}}{2-\epsilon_{\mathrm{KL}}}}$ yields, provided $\epsilon_{\mathrm{KL}}\le 1$:
    \begin{equation}
        Z_\vtheta^2\chi^2(p_*\|\Phi) \le 4Z_\vtheta^2(2-\epsilon_{\mathrm{KL}}) + 4Z_\vtheta G_\vtheta \epsilon_{\mathrm{KL}} + 4Z_\vtheta(Z_\vtheta+G_\vtheta)\sqrt{\epsilon_{\mathrm{KL}}(2-\epsilon_{\mathrm{KL}})},
    \end{equation}
    proving the second bound.
    The third and final bound follows from the fact that $\epsilon\le\sqrt{\epsilon}$ for $\epsilon\le 1$.
\end{proof}

\algVar*
\begin{proof}
    Assumption~\ref{asm:bpo} allows Algorithm~\ref{alg:bpo_theory} to query the BPO oracle at Line~\ref{line:bpo}, obtaining $\widetilde{\vtheta}_k=\mathrm{BPO}(\vtheta_k)$ with $D_{\mathrm{KL}}(p_{*,\vtheta_k}\| p_{\widetilde{\vtheta}_k})\le\epsilon_{\mathrm{KL}}$.
    So, the first statement follows immediately from Theorem~\ref{th:var_kl} and the properties of variance (just notice that $\mathbf{v}_k$ can also be written as the average of $N_{\mathrm{PG}}$ independent random variables). Then, the second statement follows by rearranging the terms and noting that:
    \begin{equation}
        N_{\mathrm{PG}}\Var_{k}\left[\widehat{\nabla} J(\vtheta_k;\mathcal{D}_{\mathrm{ON}})\right] - V_{k} = Z_k^2 - \enorm{\nabla J(\vtheta_k)}^2.
    \end{equation}
\end{proof}

\subsection{Proofs of Section~\ref{sec:theory_rate}}

For the scope of this section, fix a target policy $\vtheta^\dagger$, let $p_*$ be the corresponding optimal behavior policy $p^*_{\vtheta^\dagger}$, and let $F(\vtau)=\enorm{g_{\vtheta^\dagger}(\vtau)}$ for brevity.
Let $\widehat{L}:\vTheta\to{\Reals}_+$ be the empirical loss defined as:
\begin{equation}
    \widehat{L}(\vtheta) = -\frac{1}{n}\sum_{i=1}^nF(\vtau_i)\sum_{t=0}^{T-1}\log\pi_\vtheta(a_t^i|s_t^i),
\end{equation}
where $\vtau_i=(s_0^i,a_0^i,\dots,s_{T-1}^i,a_{T-1}^i)$, so that $\widetilde{\vtheta}=\arg\min_{\vtheta\in\vTheta}\widehat{L}(\vtheta)$. Also, let
\begin{equation}
    L(\vtheta) = \E\left[\widehat{L}(\vtheta)\right] = -\E_{\vtau\sim p_{\vtheta^\dagger}}\left[F(\vtau)\sum_{t=0}^{T-1}\log\pi_\vtheta(a_t|s_t)\right],
\end{equation}
where $\vtau=(s_0,a_0,\dots,a_{T-1},s_{T-1})$, and $\vtheta^*=\arg\min_{\vtheta\in\vTheta}L(\vtheta)$.

\begin{lemma}\label{lem:hessian}
    Under Assumptions~\ref{asm:expfam} and~\ref{asm:subgauss}:
    \begin{align}
        &\nabla L(\vtheta) = -\E_{\vtau\sim p_{\vtheta^\dagger}}\left[F(\vtau)\sum_{t=0}^{T-1} \overline{\bm{\varphi}}_\vtheta(s_t,a_t)\right],\\
        &\nabla^2 L(\vtheta) = \E_{\vtau\sim p_{\vtheta^\dagger}}\left[F(\vtau)\sum_{t=0}^{T-1} \Cov_{a\sim\pi_\vtheta(\cdot|s_t)}[\bm{\varphi}(s_t,a)] \right],\\
        &\enorm{\nabla^2L(\vtheta)} \le GT\sigma^2.
    \end{align}
\end{lemma}
\begin{proof}
    The first two statements follow immediately from Assumption~\ref{asm:expfam}. As for the third statement: 
    \begin{align}
        \enorm{\nabla^2L(\vtheta)} &\le \E\left[G \sum_{t=0}^{T-1}\enorm{\E_{a\sim\pi_\vtheta(\cdot|s_t)}\left[\overline{\bm{\varphi}}_\vtheta(s_t,a)\overline{\bm{\varphi}}_\vtheta(s_t,a)^\top\right]}\right] \\
        &\le \E\left[G \sum_{t=0}^{T-1}\E_{a\sim\pi_\vtheta(\cdot|s_t)}\left[\enorm{\overline{\bm{\varphi}}_\vtheta(s_t,a)}^2\right]\right]
        \le GT\sigma^2,
    \end{align}
    where the last inequality is by Assumption~\ref{asm:subgauss} and Proposition~\ref{prop:subgvec}.
\end{proof}

\begin{lemma}\label{lem:nabla_hat_L}
Under Assumptions~\ref{asm:expfam},~\ref{asm:real} and~\ref{asm:subgauss},
    \begin{equation}
        \E\left[\enorm{\nabla\widehat{L}(\vtheta^*)}^2\right] \le \frac{Z_{\vtheta^\dagger}GT^2\sigma^2}{n}.
    \end{equation}
\end{lemma}
\begin{proof}
    First notice that, for policies of the exponential family (Assumption~\ref{asm:expfam}):
    \begin{align}
        \E\left[\nabla\widehat{L}(\vtheta^*)\right] 
        &= \E_{\vtau\sim p_{\vtheta^\dagger}}\left[\enorm{g_{\vtheta^\dagger}(\vtau)}\sum_{t=0}^{T-1}\overline{\bm{\varphi}}_{\vtheta^*}(s_t,a_t)\right] \\
        &= Z_{\vtheta^\dagger}\E_{\vtau\sim p_{*}}\left[\sum_{t=0}^{T-1}\overline{\bm{\varphi}}_{\vtheta^*}(s_t,a_t)\right] \\
        &= Z_{\vtheta^\dagger}\E_{\vtau\sim p_{\vtheta^*}}\left[\sum_{t=0}^{T-1}\E_{a\sim\pi_{\vtheta^*}(\cdot|s_t)}\left[\overline{\bm{\varphi}}_{\vtheta^*}(s_t,a)|s_t\right]\right] \\&= 0,
    \end{align}
    where the second-to-last equality is by Assumption~\ref{asm:real}. Then
    \begin{align}
        \E\left[\enorm{\nabla\widehat{L}(\vtheta^*)}^2\right] &= \Var\left[\nabla\widehat{L}(\vtheta^*)\right] \\
        &= \frac{1}{n}\Var_{\vtau\sim p_{\vtheta^\dagger}}\left[\enorm{g_{\vtheta^\dagger}(\vtau)}\sum_{t=0}^{T-1}\overline{\bm{\varphi}}_{\vtheta^*}(s_t,a_t)\right] \\
        &= \frac{1}{n}\E_{\vtau\sim p_{\vtheta^\dagger}}\left[\enorm{g_{\vtheta^\dagger}(\vtau)}^2\enorm{\sum_{t=0}^{T-1}\overline{\bm{\varphi}}_{\vtheta^*}(s_t,a_t)}^2\right] \\
        &= \frac{Z_{\vtheta^\dagger}}{n}\E_{\vtau\sim p_{*}}\left[\enorm{g_{\vtheta^\dagger}(\vtau)}\enorm{\sum_{t=0}^{T-1}\overline{\bm{\varphi}}_{\vtheta^*}(s_t,a_t)}^2\right] \\
        &\le \frac{Z_{\vtheta^\dagger}GT}{n}\E_{\vtau\sim p_{\vtheta^*}}\left[\sum_{t=0}^{T-1}\E_{a\sim\pi_{\vtheta^*}(\cdot|s_t)}\left[\enorm{\overline{\bm{\varphi}}_{\vtheta^*}(s_t,a)}^2\middle|s_t\right]\right] \\
        &\le \frac{Z_{\vtheta^\dagger}GT^2\sigma^2}{n},
    \end{align}
    where the last inequality is by Assumption~\ref{asm:subgauss} and the second-to-last relies on Assumption~\ref{asm:real}.
\end{proof}

\mle*

\begin{proof}
By the mean value theorem, there exists a $c\in[0,1]$ such that
\begin{align}
    L(\widetilde{\vtheta}) &= L(\vtheta^*) + \langle\widetilde{\vtheta}-\vtheta^*,\nabla L(\vtheta^*)\rangle + \frac{1}{2}(\widetilde{\vtheta}-\vtheta^*)^\top \nabla^2L(\vtheta_c)(\widetilde{\vtheta}-\vtheta^*) \\
    &\le L(\vtheta^*) + \frac{1}{2} GT\sigma^2 \enorm{\widetilde{\vtheta}-\vtheta^*}^2, \label{eq:step1}
\end{align}
where $\vtheta_c=c\widetilde{\vtheta}+(1-c)\vtheta^*$ for some $c \in [0,1]$ and the last inequality is by Lemma~\ref{lem:hessian} under Assumptions~\ref{asm:expfam} and~\ref{asm:subgauss}.

Now let
\begin{equation}
    \widehat{\mathcal{G}}(\vtheta) = -\frac{1}{n}\sum_{i=1}^nF(\vtau_i)\sum_{t=0}^{T-1}\left(\overline{\bm{\varphi}}_\vtheta(s_t^i,a_t^i) - \overline{\bm{\varphi}}_{\vtheta^*}(s_t^i,a_t^i)\right),
\end{equation}
and notice that $\widehat{\mathcal{G}}(\vtheta^*)=0$, and that $\nabla \widehat{\mathcal{G}}(\vtheta) = \widehat{\mathcal{F}}(\vtheta)$ where $\widehat{\mathcal{F}}$ is defined in Assumption~\ref{asm:mineig}. Then, from the mean value theorem, there exists a $c\in[0,1]$ such that:
\begin{align}
    \widehat{\mathcal{G}}(\widetilde{\vtheta}) = \widehat{\mathcal{G}}(\vtheta^*) + (\widetilde{\vtheta}-\vtheta^*)^\top\nabla\widehat{\mathcal{G}}(\vtheta_c) = (\widetilde{\vtheta}-\vtheta^*)^\top\widehat{F}(\vtheta_c),
\end{align}
where $\vtheta_c = c\widetilde{\vtheta}+(1-c)\vtheta^*$. Hence, by Assumption~\ref{asm:mineig},
\begin{equation}
    \E\left[\enorm{\widehat{\mathcal{G}}(\widetilde{\vtheta})}^2\right] \ge \lambda_*^2\E\left[\enorm{\widetilde{\vtheta}-\vtheta^*}^2\right].\label{eq:step2}
\end{equation}

Next, notice that $\widehat{\mathcal{G}}(\vtheta) = \nabla \widehat{L}(\vtheta) - \nabla \widehat{L}(\vtheta^*)$ by Assumption~\ref{asm:expfam}. Thus, by definition of $\widetilde{\vtheta}$, $\widehat{\mathcal{G}}(\widetilde{\vtheta}) = \nabla\widehat{L}(\widetilde{\vtheta}) - \nabla\widehat{L}(\vtheta^*) = \nabla\widehat{L}(\vtheta^*)$, and
\begin{equation}
    \E\left[\enorm{\widehat{\mathcal{G}}(\widetilde{\vtheta})}^2\right] = \E\left[\enorm{\nabla\widehat{L}(\vtheta^*)}^2\right] \le \frac{Z_{\vtheta^\dagger}GT^2\sigma^2}{n}.\label{eq:step3}
\end{equation}
where the last inequality is by Lemma~\ref{lem:nabla_hat_L} under Assumptions~\ref{asm:expfam},~\ref{asm:real} and~\ref{asm:subgauss}.

Finally, chaining the inequalities from Equations~(\ref{eq:step1}), (\ref{eq:step2}), and (\ref{eq:step3}):
\begin{align}
    \E[L(\widetilde{\vtheta})] &\le L(\vtheta^*) + \frac{1}{2}GT\sigma^2\E\left[\enorm{\widetilde{\vtheta}-\vtheta^*}^2\right] \\
    &\le L(\vtheta^*) + \frac{GT\sigma^2}{2\lambda_*^2}\E\left[\enorm{\widehat{\mathcal{G}}(\widetilde{\vtheta})}^2\right] \\
    &\le L(\vtheta^*) + \frac{Z_{\vtheta^\dagger}G^2T^3\sigma^4}{2\lambda_*^2n}. \label{eq:loss_diff}
\end{align}

Finally:
\begin{align}
    D_\mathrm{KL}(p_*\|p_{\widetilde{\vtheta}})
    &=\E_{\vtau\sim p_*}\left[\log p_*(\vtau)-\log p_{\widetilde{\vtheta}}(\vtau)\right] \\
    &=\E_{\vtau\sim p_{\vtheta^*}}\left[\log p_{\vtheta^*}(\vtau)-\log p_{\widetilde{\vtheta}}(\vtau)\right] \\
    &=\E_{\vtau\sim p_{\vtheta^*}}\left[\sum_{t=0}^{T-1}\left(\log \pi_{\vtheta^*}(a_t|s_t)-\log \pi_{\widetilde{\vtheta}}(a_t|s_t)\right)\right] \\
    &=\frac{1}{Z_{\vtheta^\dagger}}\E_{\vtau\sim p_{\vtheta^\dagger}}\left[\enorm{g_{\vtheta^\dagger}(\vtau)}\sum_{t=0}^{T-1}\left(\log \pi_{\vtheta^*}(a_t|s_t)-\log \pi_{\widetilde{\vtheta}}(a_t|s_t)\right)\right] \\
    &=\frac{L(\widetilde{\vtheta}) - L(\vtheta^*)}{Z_{\vtheta^\dagger}},
\end{align}
and by Equation~(\ref{eq:loss_diff}):
\begin{equation}
    \E[D_\mathrm{KL}(p_*\|p_{\widetilde{\vtheta}})] 
    = \frac{\E[L(\widetilde{\vtheta})] - L(\vtheta^*)}{Z_{\vtheta^\dagger}}
    \le \frac{G^2T^3\sigma^4}{2\lambda_*^2n}.
\end{equation}

\end{proof}

\begin{lemma}\label{lem:quadratic_bound}
    Under Assumptions~\ref{asm:expfam} and~\ref{asm:subgauss}, for all $\vtheta,\vtheta'\in\vTheta$:
    \begin{equation*}
        J(\vtheta') - J(\vtheta) \ge \langle \vtheta'-\vtheta,\nabla J(\vtheta)\rangle - \frac{R_{\max}\sigma^2}{(1-\gamma)^2}\enorm{\vtheta'-\vtheta}^2.
    \end{equation*}
\end{lemma}
\begin{proof}
    Under Assumption~\ref{asm:expfam},
    \begin{align*}
    \E_{a\sim\pi_{\vtheta}(\cdot|s)}[\enorm{\nabla \log \pi_{\vtheta}(a|s)}^2] &=
    \E_{a\sim\pi_{\vtheta}(\cdot|s)}[\enorm{\overline{\bm{\varphi}}_\vtheta(s,a)}^2] \le \sigma^2,
    \end{align*}
    where the last inequality is by sub-Gaussianity of the centered sufficient statistic (Assumption~\ref{asm:subgauss} and Proposition~\ref{prop:subgvec}). Similarly:
    \begin{align*}
    \E_{a\sim\pi_{\vtheta}(\cdot|s)}[\enorm{\nabla^2 \log \pi_{\vtheta}}] &=
    \enorm{\Cov_{a\sim\pi_{\vtheta}(\cdot|s)}[\bm{\varphi}(s,a)]} \\
    &\le \mathrm{trace}\left(\Cov_{a\sim\pi_{\vtheta}(\cdot|s)}[\bm{\varphi}(s,a)]\right) \\
    &=\Var_{a\sim\pi_{\vtheta}(\cdot|s)}[\bm{\varphi}(s,a)]\\
    &=\E_{a\sim\pi_{\vtheta}(\cdot|s)}[\enorm{\overline{\bm{\varphi}}_\vtheta(s,a)}^2] \le \sigma^2.
    \end{align*}
    Hence, by Proposition~\ref{prop:els}, $\enorm{\nabla^2J(\vtheta)}\le\frac{2R_{\max}\sigma^2}{(1-\gamma)^2}$ for all $\vtheta\in\vTheta$. Finally, by the mean value theorem, there exists $c\in(0,1)$ such that:
    \begin{align*}
        J(\vtheta') &= J(\vtheta) + \langle \vtheta'-\vtheta,\nabla J(\vtheta)\rangle + \frac{1}{2}(\vtheta'-\vtheta)^\top \nabla^2J(\vtheta_c)(\vtheta'-\vtheta) \\
        &\ge J(\vtheta) + \langle \vtheta'-\vtheta,\nabla J(\vtheta)\rangle - \frac{1}{2}\enorm{\nabla^2 J(\vtheta)}\enorm{\vtheta'-\vtheta}^2 \\
        &\ge J(\vtheta) + \langle \vtheta'-\vtheta,\nabla J(\vtheta)\rangle - \frac{R_{\max}\sigma^2}{(1-\gamma)^2}\enorm{\vtheta'-\vtheta}^2,
    \end{align*}
    where $\vtheta_c=c\vtheta+(1-c)\vtheta'$ for some $c \in [0,1]$.
\end{proof}

\fullvar*
\begin{proof}
    The first statement follows from Theorem~\ref{th:alg_var} and Lemma~\ref{lem:mle}. For the second statement, notice that for every $\beta\in(0,1)$ there is an $\epsilon\in(0,1)$ such that $\beta=\sqrt{\epsilon/(2-\epsilon)}$. The assumption on the batch size $N$ guarantees that $\epsilon$ is a valid upper bound on the KL divergence.
\end{proof}

\rate*
\begin{proof}
    By Lemma~\ref{lem:quadratic_bound}:
    \begin{align}
        &\E_k[J(\vtheta_{k+1} - J(\vtheta_k))] \ge \E_k\left[\langle \vtheta_{k+1}-\vtheta_k,\nabla J(\vtheta_k)\rangle - \frac{R_{\max}\sigma^2}{(1-\gamma)^2}\enorm{\vtheta_{k+1}-\vtheta_k}^2\right] \\
        &\qquad=\E_k\left[\alpha\langle \bm{v}_k,\nabla J(\vtheta_k)\rangle - \frac{\alpha^2R_{\max}\sigma^2}{(1-\gamma)^2}\enorm{\bm{v}_k}^2\right] \\
        &\qquad=\alpha\enorm{\nabla J(\vtheta_k)}^2 - \frac{\alpha^2R_{\max}\sigma^2}{(1-\gamma)^2}E_k[\enorm{\bm{v}_k}^2] \\
        &\qquad=\alpha\left(1-\frac{\alpha R_{\max}\sigma^2}{(1-\gamma)^2}\right)\enorm{\nabla J(\vtheta_k)}^2 - \frac{\alpha^2R_{\max}\sigma^2}{(1-\gamma^2)}\Var_k[\bm{v}_k]\\
        &\qquad\ge 
        \alpha\left(1-\frac{\alpha R_{\max}\sigma^2}{(1-\gamma)^2}\right)\enorm{\nabla J(\vtheta_k)}^2 - \frac{\alpha^2R_{\max}\sigma^2(18Z_k^2-\enorm{\nabla J(\vtheta_k)}^2)}{(1-\gamma)^2N} \\&\qquad\qquad-\frac{\alpha^2R_{\max}\sigma^4Z_k(Z_k+2G)GT^{3/2}}{2\lambda_*(1-\gamma)^2N^{3/2}}\\
        &\qquad\ge 
        \alpha\left(1-\frac{\alpha R_{\max}\sigma^2}{(1-\gamma)^2}\right)\enorm{\nabla J(\vtheta_k)}^2 - \frac{\alpha^2R_{\max}\sigma^2\widetilde{V}}{(1-\gamma)^2N} \\&\qquad\qquad-\frac{\alpha^2R_{\max}\sigma^4Z_k(Z_k+2G)GT^{3/2}}{2\lambda_*(1-\gamma)^2N^{3/2}}
        \\
        &\qquad\ge 
        \alpha\left(1-\frac{\alpha R_{\max}\sigma^2}{(1-\gamma)^2}\right)\enorm{\nabla J(\vtheta_k)}^2 - \frac{\alpha^2R_{\max}\sigma^2\widetilde{V}}{(1-\gamma)^2N} \\&\qquad\qquad-\frac{\alpha^2R_{\max}\sigma^4Z_k(Z_k+2G)GT^{3/2}}{2\lambda_*(1-\gamma)^2N^{3/2}} \\
        &\qquad\ge 
        \alpha\left(1-\frac{\alpha R_{\max}\sigma^2}{(1-\gamma)^2}\right)\enorm{\nabla J(\vtheta_k)}^2 - \frac{\alpha^2R_{\max}\sigma^2\widetilde{V}}{(1-\gamma)^2N} \\&\qquad\qquad-\frac{\alpha^2R_{\max}^4\sigma^5\norm[\infty][]{\bm{\varphi}}(\sqrt{T}\sigma+2T\norm[\infty][]{\bm{\varphi}})T^{3}}{2\lambda_*(1-\gamma)^5N^{3/2}}. 
    \end{align}
    Summing both sides for $k=0,\dots,K-1$, by the tower rule of expectation, the sum on the LHS telescopes:
    \begin{align}
        \E[J(\vtheta_K)] - J(\vtheta_0) \ge \alpha\left(1-\alpha C_1\right)\E\left[\sum_{k=0}^{K-1}\enorm{\nabla J(\vtheta_k)}^2\right] - \frac{K\alpha^2 C_1\widetilde{V}}{N} - \frac{K\alpha^2 C_2}{N^{3/2}}.
    \end{align}
    Rearranging and dividing by $K$, by definition of $\vtheta_{\mathrm{OUT}}$, provided $\alpha<1/C_1$:
    \begin{align}
        \E\left[\enorm{\nabla J(\vtheta_{\mathrm{OUT}})}^2\right] &\le \frac{\E[J(\vtheta_K)] - J(\vtheta_0)}{\alpha(1-\alpha C_1)K} + \frac{\alpha C_1\widetilde{V}}{(1-\alpha C_1)N} + \frac{\alpha C_2}{(1-\alpha C_1)N^{3/2}} \\
        &\le \frac{J(\vtheta^*) - J(\vtheta_0)}{\alpha(1-\alpha C_1)K} + \frac{\alpha C_1\widetilde{V}}{(1-\alpha C_1)N} + \frac{\alpha C_2}{(1-\alpha C_1)N^{3/2}}.
    \end{align}
    Now let $N=\epsilon^{-4/3}$ and $\alpha=\min\left\{\frac{1}{2C_1}, \frac{\epsilon^{2/3}}{6C_1\widetilde{V}}, \frac{1}{6C_2}\right\}$. Then:
    \begin{align}
        \E\left[\enorm{\nabla J(\vtheta_{\mathrm{OUT}})}^2\right] &\le
        \frac{2(J(\vtheta^*) - J(\vtheta_0))}{\alpha K} + 2\alpha C_1\widetilde{V}\epsilon^{4/3} + 2\alpha C_2\epsilon^2.
    \end{align}
    We consider three cases, and call $\overline{K}$ the smallest integer $K$ such that $\E\left[\enorm{\nabla J(\vtheta_{\mathrm{OUT}})}^2\right]\le\epsilon^2$. Note that the latter implies $\E\left[\enorm{\nabla J(\vtheta_{\mathrm{OUT}})}\right]\le\epsilon$ by Jensen's inequality.
    
    \paragraph{Case 1.} Suppose $\frac{1}{2C_1}\le\min\left\{\frac{\epsilon^{2/3}}{6C_1\widetilde{V}}, \frac{1}{6C_2}\right\}$. Then $\alpha=\frac{1}{2C_1}$ and
    \begin{align}
         \E\left[\enorm{\nabla J(\vtheta_{\mathrm{OUT}})}^2\right] &\le \frac{4C_1(J(\vtheta^*) - J(\vtheta_0))}{K} + \widetilde{V}\epsilon^{4/3} + \frac{C_2\epsilon^2}{C_1} \\
         &\le \frac{4C_1(J(\vtheta^*) - J(\vtheta_0))}{K} + \frac{\epsilon^{2}}{3} + \frac{\epsilon^2}{3},
    \end{align}
    so $\overline{K}\le\frac{12C_1(J(\vtheta^*) - J(\vtheta_0))}{\epsilon^2}$ in this case.

    \paragraph{Case 2.} Suppose $\frac{\epsilon^{2/3}}{6C_1\widetilde{V}}\le\min\left\{\frac{1}{2C_1}, \frac{1}{6C_2}\right\}$. Then $\alpha=\frac{\epsilon^{2/3}}{6C_1\widetilde{V}}$ and
    \begin{align}
        \E\left[\enorm{\nabla J(\vtheta_{\mathrm{OUT}})}^2\right] &\le \frac{12C_1\widetilde{V}(J(\vtheta^*) - J(\vtheta_0))}{\epsilon^{2/3} K} + \frac{\epsilon^2}{3} + \frac{C_2\epsilon^{8/3}}{3C_1\widetilde{V}} \\
        &\le \frac{12C_1\widetilde{V}(J(\vtheta^*) - J(\vtheta_0))}{\epsilon^{2/3} K} + \frac{\epsilon^2}{3} + \frac{\epsilon^{2}}{3},
    \end{align}
    so $\overline{K}\le\frac{36C_1\widetilde{V}(J(\vtheta^*) - J(\vtheta_0))}{\epsilon^{8/3}}$ in this case.

    \paragraph{Case 3.} Suppose $\frac{1}{6C_2}\le\min\left\{\frac{1}{2C_1}, \frac{\epsilon^{2/3}}{6C_1\widetilde{V}}\right\}$. Then $\alpha=\frac{1}{6C_2}$ and
    \begin{align}
        \E\left[\enorm{\nabla J(\vtheta_{\mathrm{OUT}})}^2\right] &\le
        \frac{12C_2(J(\vtheta^*) - J(\vtheta_0))}{K} + \frac{C_1\widetilde{V}\epsilon^{4/3}}{3C_2} + \frac{\epsilon^2}{3} \\
        &\le \frac{12C_2(J(\vtheta^*) - J(\vtheta_0))}{K} + \frac{\epsilon^2}{3} + \frac{\epsilon^2}{3},
    \end{align}
    so $\overline{K}\le\frac{36C_2(J(\vtheta^*) - J(\vtheta_0))}{\epsilon^2}$ in this case.

    Considering the three cases, we know for sure that
    \begin{equation}
        \overline{K} \le 12(J(\vtheta^*)-J(\vtheta_0)) \left(\frac{3C_1\widetilde{V}}{\epsilon^{8/3}}+\frac{C_1+3C_2}{\epsilon^2}\right).
    \end{equation}
    So the total number of trajectories is at most
    \begin{equation}
        N\overline{K} = \epsilon^{-4/3}\overline{K} \le 12(J(\vtheta^*)-J(\vtheta_0)) \left(\frac{3C_1\widetilde{V}}{\epsilon^{4}}+\frac{C_1+3C_2}{\epsilon^{10/3}}\right).
    \end{equation}
\end{proof}

\section{Auxiliary Results}
\begin{proposition}\label{prop:subgvec}
    Let $\bm{X}$ be a zero-mean $\sigma$-subgaussian random vector in $\Reals^d$ in the sense of Assumption~\ref{asm:subgauss}.
    Then
    \begin{equation*}
        \E\left[\enorm{\bm{X}}^2\right] \le \sigma^2.
    \end{equation*}
\end{proposition}
\begin{proof}
    For any $\lambda > 0$ and $t\in\Reals^d$ with $\enorm{\bm{t}}=1$, by hypothesis, $\E[\exp(\lambda \bm{t}^\top \bm{X})]\le\exp(\lambda^2\sigma^2/2)$. Then
    \begin{equation}
        1 + \lambda \bm{t}^\top\E[\bm{X}] + \frac{\lambda^2}{2}  \E[(\bm{t}^\top \bm{X})^2] + o(\lambda^2) \le 1 + \frac{\lambda^2\sigma^2}{2} + o(\lambda^2),
    \end{equation}
    so $\E[(\bm{t}^\top \bm{X})^2] \le \sigma^2$. The proof is concluded by noting that $\enorm{\bm{X}}=\sup_{\bm{t}\in\Reals^d : \enorm{\bm{t}}=1}\{\bm{t}^\top \bm{X}\}$.
\end{proof}

\begin{proposition}[Lemma 4.4 from~\cite{yuan2022general}]\label{prop:els}
    If there are constants $L_1,L_2>0$ such that the following holds for all $\vtheta\in\vTheta$ and $s\in\mathcal{S}$ (E-LS, Assumption 4.1 in\citet{yuan2022general}):
    \begin{align}
        &\E_{a\sim\pi_{\vtheta}(\cdot|s)}[\enorm{\nabla \log \pi_{\vtheta}(a|s)}^2] \le L_1^2,\\
        &\E_{a\sim\pi_{\vtheta}(\cdot|s)}[\enorm{\nabla^2 \log \pi_{\vtheta}(a|s)}] \le L_2,
    \end{align}
    then $\enorm{\nabla^2 J(\vtheta)} \le \frac{R_{\max}(L_1^2+L_2)}{(1-\gamma)^2}$ for all $\vtheta\in\vTheta$.
\end{proposition}

\section{Additional numerical results}
\label{sec:appendix_numerical_results}
In this section, we report the full experimental results of Section~\ref{sec:numerical}, with different target policy parameters (Table~\ref{tab:var_lq_mu_appendix}), standard deviations (Table~\ref{tab:var_lq_std_appendix}), LQ horizons (Table~\ref{tab:var_lq_horizon_appendix}) and state dimensions (Table~\ref{tab:var_lq_dimensions_appendix}). Each experiment was repeated 100 times and run with different hyper-parameters of our off-policy method, i.e., the defensive coefficient $\beta$, the biased off-policy practical gradient calculation (the offline estimation of the KL divergence here is not possible), and the batch sizes $N_{\mathrm{BPO}}$ and $N_{\mathrm{PG}}$.

\begin{longtable}{S S S c >$c<$ >$c<$ >$c<$ S[round-precision=1]}
\label{tab:var_lq_mu_appendix}\\
\caption{LQ environment, with horizon = 2 and state dimension = 1, and target policy with $\log \sigma = 0$. Variance reduction in off-policy gradient, expressed as $\Delta \! \Var$ and its 95\% Gaussian confidence interval $(\Delta \! \Var^-,\Delta \! \Var^+)$, with different hyper-parameters and values of $\vtheta$.}\\
\sisetup{ round-mode = places, round-precision = 2 }\\
\hline
\text{$\Delta \! \Var$} & \text{$\Delta \! \Var^-$} & \text{$\Delta \! \Var^+$} & biased & \beta & N_{\mathrm{BPO}} & N_{\mathrm{PG}} & $\vtheta$\\
\hline
0.311033  & -0.068349 & 0.690415  & False & 0.0 & 10 & 90 & -1.0 \\
0.209282  & -0.216109 & 0.634673  & False & 0.0 & 50 & 50 & -1.0 \\
0.321055  & -0.169663 & 0.811773  & False & 0.4 & 10 & 90 & -1.0 \\
0.306358  & -0.159384 & 0.772100  & False & 0.4 & 50 & 50 & -1.0 \\
0.290852  & -0.136284 & 0.717988  & False & 0.8 & 10 & 90 & -1.0 \\
0.029209  & -0.385136 & 0.443554  & False & 0.8 & 50 & 50 & -1.0 \\
0.508645  & 0.032941  & 0.984350  & True  & 0.0 & 10 & 90 & -1.0 \\
0.703738  & 0.253894  & 1.153583  & True  & 0.0 & 30 & 70 & -1.0 \\
0.398966  & 0.183075  & 0.614856  & True  & 0.0 & 50 & 50 & -1.0 \\
0.270046  & -0.144759 & 0.684851  & True  & 0.4 & 10 & 90 & -1.0 \\
0.469772  & 0.258537  & 0.681006  & True  & 0.4 & 50 & 50 & -1.0 \\
0.235018  & -0.137044 & 0.607080  & True  & 0.8 & 10 & 90 & -1.0 \\
0.561355  & 0.302557  & 0.820153  & True  & 0.8 & 50 & 50 & -1.0 \\
0.140721  & 0.006513  & 0.274928  & False & 0.0 & 10 & 90 & -0.5 \\
0.106241  & -0.011837 & 0.224319  & False & 0.0 & 30 & 70 & -0.5 \\
0.004764  & -0.112903 & 0.122432  & False & 0.0 & 50 & 50 & -0.5 \\
0.111122  & -0.034218 & 0.256462  & False & 0.4 & 10 & 90 & -0.5 \\
-0.027326 & -0.127503 & 0.072851  & False & 0.4 & 50 & 50 & -0.5 \\
0.037222  & -0.083851 & 0.158295  & False & 0.8 & 10 & 90 & -0.5 \\
-0.050209 & -0.168186 & 0.067768  & False & 0.8 & 50 & 50 & -0.5 \\
0.047626  & -0.069773 & 0.165025  & True  & 0.0 & 10 & 90 & -0.5 \\
0.220818  & 0.068769  & 0.372868  & True  & 0.0 & 50 & 50 & -0.5 \\
-0.016716 & -0.179981 & 0.146548  & True  & 0.4 & 10 & 90 & -0.5 \\
0.222632  & 0.064101  & 0.381162  & True  & 0.4 & 50 & 50 & -0.5 \\
0.078851  & -0.041082 & 0.198785  & True  & 0.8 & 10 & 90 & -0.5 \\
0.195087  & 0.059070  & 0.331105  & True  & 0.8 & 50 & 50 & -0.5 \\
0.055112  & -0.037841 & 0.148065  & False & 0.0 & 10 & 90 & 0.0 \\
-0.025057 & -0.207522 & 0.157408  & False & 0.0 & 30 & 70 & 0.0 \\
-0.093295 & -0.238400 & 0.051810  & False & 0.0 & 50 & 50 & 0.0 \\
-0.055413 & -0.187053 & 0.076227  & False & 0.4 & 10 & 90 & 0.0 \\
-0.134235 & -0.228541 & -0.039929 & False & 0.4 & 50 & 50 & 0.0 \\
-0.044929 & -0.146132 & 0.056273  & False & 0.8 & 10 & 90 & 0.0 \\
-0.031705 & -0.121440 & 0.058030  & False & 0.8 & 30 & 70 & 0.0 \\
-0.144370 & -0.240908 & -0.047833 & False & 0.8 & 50 & 50 & 0.0 \\
0.063952  & -0.017625 & 0.145529  & True  & 0.0 & 10 & 90 & 0.0 \\
0.120536  & 0.057602  & 0.183469  & True  & 0.0 & 50 & 50 & 0.0 \\
0.044606  & -0.063858 & 0.153070  & True  & 0.4 & 10 & 90 & 0.0 \\
0.094860  & 0.012727  & 0.176992  & True  & 0.4 & 50 & 50 & 0.0 \\
0.035522  & -0.039497 & 0.110541  & True  & 0.8 & 10 & 90 & 0.0 \\
0.120686  & 0.060903  & 0.180469  & True  & 0.8 & 50 & 50 & 0.0 \\
0.392953  & -0.018980 & 0.804886  & False & 0.0 & 10 & 90 & 0.5 \\
0.122100  & -0.185945 & 0.430145  & False & 0.0 & 50 & 50 & 0.5 \\
-0.053468 & -0.408500 & 0.301563  & False & 0.4 & 10 & 90 & 0.5 \\
-0.094985 & -0.448332 & 0.258362  & False & 0.4 & 50 & 50 & 0.5 \\
0.058454  & -0.334086 & 0.450995  & False & 0.8 & 10 & 90 & 0.5 \\
-0.233754 & -0.643237 & 0.175729  & False & 0.8 & 30 & 70 & 0.5 \\
-0.285911 & -0.647637 & 0.075815  & False & 0.8 & 50 & 50 & 0.5 \\
0.217064  & -0.100778 & 0.534905  & True  & 0.0 & 10 & 90 & 0.5 \\
0.324804  & 0.159038  & 0.490571  & True  & 0.0 & 30 & 70 & 0.5 \\
0.204845  & -0.114787 & 0.524477  & True  & 0.0 & 50 & 50 & 0.5 \\
0.084464  & -0.244899 & 0.413827  & True  & 0.4 & 10 & 90 & 0.5 \\
0.408988  & 0.144608  & 0.673367  & True  & 0.4 & 50 & 50 & 0.5 \\
0.177405  & -0.197537 & 0.552347  & True  & 0.8 & 10 & 90 & 0.5 \\
0.296821  & 0.071193  & 0.522449  & True  & 0.8 & 50 & 50 & 0.5 \\
1.388987  & 0.323475  & 2.454499  & False & 0.0 & 10 & 90 & 1.0 \\
0.562928  & -0.714201 & 1.840058  & False & 0.0 & 50 & 50 & 1.0 \\
0.006273  & -1.342498 & 1.355045  & False & 0.4 & 10 & 90 & 1.0 \\
-1.602914 & -3.087272 & -0.118555 & False & 0.4 & 50 & 50 & 1.0 \\
0.163557  & -1.012538 & 1.339652  & False & 0.8 & 10 & 90 & 1.0 \\
-1.083920 & -2.889235 & 0.721395  & False & 0.8 & 50 & 50 & 1.0 \\
1.643050  & -0.103186 & 3.389286  & True  & 0.0 & 10 & 90 & 1.0 \\
1.260688  & 0.628243  & 1.893133  & True  & 0.0 & 50 & 50 & 1.0 \\
-0.856033 & -2.625640 & 0.913575  & True  & 0.4 & 10 & 90 & 1.0 \\
1.503771  & 0.775425  & 2.232117  & True  & 0.4 & 50 & 50 & 1.0 \\
1.148023  & -0.616740 & 2.912785  & True  & 0.8 & 10 & 90 & 1.0 \\
2.048126  & 1.127738  & 2.968514  & True  & 0.8 & 50 & 50 & 1.0 \\
\bottomrule
\end{longtable}

\begin{longtable}{S S S c >$c<$ >$c<$ >$c<$ S[round-precision=1]}
\label{tab:var_lq_std_appendix}\\
\caption{LQ environment, with horizon = 2 and state dimension = 1, and target policy with $\vtheta = 0$. Variance reduction in off-policy gradient, expressed as $\Delta \! \Var$ and its 95\% Gaussian confidence interval $(\Delta \! \Var^-,\Delta \! \Var^+)$, with different hyper-parameters and values of $\log \sigma$.}\\

\sisetup{ round-mode = places, round-precision = 3 }\\

\hline
\text{$\Delta \! \Var$} & \text{$\Delta \! \Var^-$} & \text{$\Delta \! \Var^+$}& biased & \beta & N_{\mathrm{BPO}} & N_{\mathrm{PG}} & $\log \sigma$\\
\hline
\endfirsthead
\endlastfoot

-0.005930 & -0.029759 & 0.017899  & False & 0.0 & 10 & 90 & -1.0 \\
0.005660  & -0.009167 & 0.020486  & False & 0.0 & 30 & 70 & -1.0 \\
-0.012477 & -0.029328 & 0.004374  & False & 0.0 & 50 & 50 & -1.0 \\
-0.019162 & -0.045400 & 0.007076  & False & 0.4 & 10 & 90 & -1.0 \\
-0.009285 & -0.022311 & 0.003742  & False & 0.4 & 30 & 70 & -1.0 \\
-0.031216 & -0.049090 & -0.013342 & False & 0.4 & 50 & 50 & -1.0 \\
0.003573  & -0.007268 & 0.014413  & False & 0.8 & 10 & 90 & -1.0 \\
-0.001659 & -0.016295 & 0.012977  & False & 0.8 & 30 & 70 & -1.0 \\
-0.019955 & -0.042961 & 0.003050  & False & 0.8 & 50 & 50 & -1.0 \\
0.007892  & -0.012674 & 0.028458  & True  & 0.0 & 10 & 90 & -1.0 \\
0.003067  & -0.012134 & 0.018267  & True  & 0.0 & 30 & 70 & -1.0 \\
0.022473  & 0.009903  & 0.035043  & True  & 0.0 & 50 & 50 & -1.0 \\
0.009122  & -0.008918 & 0.027162  & True  & 0.4 & 10 & 90 & -1.0 \\
0.011459  & -0.002588 & 0.025507  & True  & 0.4 & 30 & 70 & -1.0 \\
0.024397  & 0.013399  & 0.035394  & True  & 0.4 & 50 & 50 & -1.0 \\
0.008570  & -0.006337 & 0.023477  & True  & 0.8 & 10 & 90 & -1.0 \\
0.013768  & 0.000079  & 0.027457  & True  & 0.8 & 30 & 70 & -1.0 \\
0.021262  & 0.006614  & 0.035910  & True  & 0.8 & 50 & 50 & -1.0 \\
0.008268  & -0.029189 & 0.045725  & False & 0.0 & 10 & 90 & -0.5 \\
-0.005033 & -0.033272 & 0.023205  & False & 0.0 & 30 & 70 & -0.5 \\
-0.012141 & -0.053830 & 0.029549  & False & 0.0 & 50 & 50 & -0.5 \\
-0.010019 & -0.055467 & 0.035429  & False & 0.4 & 10 & 90 & -0.5 \\
-0.041924 & -0.082961 & -0.000888 & False & 0.4 & 30 & 70 & -0.5 \\
-0.017607 & -0.063668 & 0.028453  & False & 0.4 & 50 & 50 & -0.5 \\
0.005316  & -0.024194 & 0.034827  & False & 0.8 & 10 & 90 & -0.5 \\
-0.013986 & -0.039499 & 0.011528  & False & 0.8 & 30 & 70 & -0.5 \\
-0.036312 & -0.075203 & 0.002580  & False & 0.8 & 50 & 50 & -0.5 \\
-0.020111 & -0.084565 & 0.044343  & True  & 0.0 & 10 & 90 & -0.5 \\
0.015060  & -0.035335 & 0.065454  & True  & 0.0 & 30 & 70 & -0.5 \\
0.043786  & 0.024192  & 0.063380  & True  & 0.0 & 50 & 50 & -0.5 \\
0.006319  & -0.025938 & 0.038576  & True  & 0.4 & 10 & 90 & -0.5 \\
0.026350  & -0.001615 & 0.054315  & True  & 0.4 & 30 & 70 & -0.5 \\
0.047319  & 0.022475  & 0.072162  & True  & 0.4 & 50 & 50 & -0.5 \\
-0.008211 & -0.035239 & 0.018817  & True  & 0.8 & 10 & 90 & -0.5 \\
0.033399  & 0.010778  & 0.056020  & True  & 0.8 & 30 & 70 & -0.5 \\
0.039081  & 0.017615  & 0.060547  & True  & 0.8 & 50 & 50 & -0.5 \\
0.055784  & -0.044796 & 0.156364  & False & 0.0 & 10 & 90 & 0.0 \\
0.041080  & -0.035048 & 0.117208  & False & 0.0 & 30 & 70 & 0.0 \\
-0.005922 & -0.126384 & 0.114540  & False & 0.0 & 50 & 50 & 0.0 \\
-0.108877 & -0.277607 & 0.059853  & False & 0.4 & 10 & 90 & 0.0 \\
-0.068676 & -0.192034 & 0.054683  & False & 0.4 & 30 & 70 & 0.0 \\
-0.014881 & -0.095332 & 0.065571  & False & 0.4 & 50 & 50 & 0.0 \\
-0.023581 & -0.100187 & 0.053025  & False & 0.8 & 10 & 90 & 0.0 \\
-0.019248 & -0.106283 & 0.067788  & False & 0.8 & 30 & 70 & 0.0 \\
-0.110193 & -0.239846 & 0.019460  & False & 0.8 & 50 & 50 & 0.0 \\
-0.012017 & -0.095599 & 0.071565  & True  & 0.0 & 10 & 90 & 0.0 \\
0.065525  & -0.001559 & 0.132608  & True  & 0.0 & 30 & 70 & 0.0 \\
0.103235  & 0.046921  & 0.159549  & True  & 0.0 & 50 & 50 & 0.0 \\
0.010584  & -0.077404 & 0.098572  & True  & 0.4 & 10 & 90 & 0.0 \\
0.043050  & -0.036176 & 0.122275  & True  & 0.4 & 30 & 70 & 0.0 \\
0.129326  & 0.022730  & 0.235923  & True  & 0.4 & 50 & 50 & 0.0 \\
0.033573  & -0.057323 & 0.124468  & True  & 0.8 & 10 & 90 & 0.0 \\
0.042062  & -0.015675 & 0.099800  & True  & 0.8 & 30 & 70 & 0.0 \\
0.124324  & 0.048137  & 0.200510  & True  & 0.8 & 50 & 50 & 0.0 \\
-0.033518 & -0.456779 & 0.389743  & False & 0.0 & 10 & 90 & 0.5 \\
0.151897  & -0.199452 & 0.503245  & False & 0.0 & 30 & 70 & 0.5 \\
0.157868  & -0.245560 & 0.561297  & False & 0.0 & 50 & 50 & 0.5 \\
-0.261459 & -0.687112 & 0.164193  & False & 0.4 & 10 & 90 & 0.5 \\
-0.044700 & -0.398698 & 0.309298  & False & 0.4 & 30 & 70 & 0.5 \\
-0.136862 & -0.578197 & 0.304473  & False & 0.4 & 50 & 50 & 0.5 \\
-0.160264 & -0.574757 & 0.254229  & False & 0.8 & 10 & 90 & 0.5 \\
-0.293061 & -0.759271 & 0.173149  & False & 0.8 & 30 & 70 & 0.5 \\
-0.288713 & -0.862787 & 0.285361  & False & 0.8 & 50 & 50 & 0.5 \\
0.161052  & -0.278498 & 0.600601  & True  & 0.0 & 10 & 90 & 0.5 \\
0.148964  & -0.147220 & 0.445148  & True  & 0.0 & 30 & 70 & 0.5 \\
0.556353  & 0.215147  & 0.897560  & True  & 0.0 & 50 & 50 & 0.5 \\
0.105981  & -0.300877 & 0.512839  & True  & 0.4 & 10 & 90 & 0.5 \\
0.227993  & -0.026717 & 0.482703  & True  & 0.4 & 30 & 70 & 0.5 \\
0.483820  & 0.186378  & 0.781262  & True  & 0.4 & 50 & 50 & 0.5 \\
0.240989  & -0.039873 & 0.521851  & True  & 0.8 & 10 & 90 & 0.5 \\
0.419434  & 0.145579  & 0.693288  & True  & 0.8 & 30 & 70 & 0.5 \\
0.590495  & 0.244142  & 0.936848  & True  & 0.8 & 50 & 50 & 0.5 \\
1.535046  & -0.378748 & 3.448839  & False & 0.0 & 10 & 90 & 1.0 \\
1.186207  & -0.690749 & 3.063163  & False & 0.0 & 30 & 70 & 1.0 \\
0.581094  & -1.889402 & 3.051590  & False & 0.0 & 50 & 50 & 1.0 \\
-0.436245 & -2.535319 & 1.662828  & False & 0.4 & 10 & 90 & 1.0 \\
0.392720  & -1.539439 & 2.324879  & False & 0.4 & 30 & 70 & 1.0 \\
-0.407481 & -2.796212 & 1.981250  & False & 0.4 & 50 & 50 & 1.0 \\
-0.025073 & -2.070740 & 2.020595  & False & 0.8 & 10 & 90 & 1.0 \\
0.604685  & -1.277262 & 2.486632  & False & 0.8 & 30 & 70 & 1.0 \\
-2.374359 & -5.105622 & 0.356903  & False & 0.8 & 50 & 50 & 1.0 \\
2.055510  & 0.523027  & 3.587994  & True  & 0.0 & 10 & 90 & 1.0 \\
3.247087  & 1.631413  & 4.862761  & True  & 0.0 & 30 & 70 & 1.0 \\
3.176471  & 1.951825  & 4.401116  & True  & 0.0 & 50 & 50 & 1.0 \\
-0.638350 & -2.931465 & 1.654765  & True  & 0.4 & 10 & 90 & 1.0 \\
2.361232  & -0.389329 & 5.111792  & True  & 0.4 & 30 & 70 & 1.0 \\
3.773828  & 2.399339  & 5.148317  & True  & 0.4 & 50 & 50 & 1.0 \\
-0.121002 & -2.025865 & 1.783861  & True  & 0.8 & 10 & 90 & 1.0 \\
2.700678  & 0.718395  & 4.682962  & True  & 0.8 & 30 & 70 & 1.0 \\
4.041229  & 2.016358  & 6.066101  & True  & 0.8 & 50 & 50 & 1.0 \\

\bottomrule
\end{longtable}

\begin{longtable}{S S S c >$c<$ >$c<$ >$c<$ S[round-precision=1]}
\label{tab:var_lq_horizon_appendix}\\
\caption{LQ environment, with state dimension = 1, and target policy with $\vtheta = 0$ and $\log \sigma = 0$. Variance reduction in off-policy gradient, expressed as $\Delta \! \Var$ and its 95\% Gaussian confidence interval $(\Delta \! \Var^-,\Delta \! \Var^+)$, with different hyper-parameters and values of LQ horizon.}\\

\sisetup{ round-mode = places, round-precision = 3 }\\

\hline
\text{$\Delta \! \Var$} & \text{$\Delta \! \Var^-$} & \text{$\Delta \! \Var^+$} & biased & \beta & N_{\mathrm{BPO}} & N_{\mathrm{PG}} & horizon\\
\hline
\endfirsthead
\endlastfoot
    
0.069930    & -0.046726   & 0.186585   & False & 0.0 & 10 & 90 & 2 \\
0.041136    & -0.072254   & 0.154527   & False & 0.0 & 30 & 70 & 2 \\
-0.005922   & -0.126384   & 0.114540   & False & 0.0 & 50 & 50 & 2 \\
-0.050883   & -0.162004   & 0.060239   & False & 0.4 & 10 & 90 & 2 \\
0.010338    & -0.076535   & 0.097211   & False & 0.4 & 30 & 70 & 2 \\
-0.090330   & -0.192410   & 0.011749   & False & 0.4 & 50 & 50 & 2 \\
0.035092    & -0.055714   & 0.125898   & False & 0.8 & 10 & 90 & 2 \\
-0.007530   & -0.102390   & 0.087330   & False & 0.8 & 30 & 70 & 2 \\
-0.115648   & -0.213301   & -0.017995  & False & 0.8 & 50 & 50 & 2 \\
0.066612    & -0.001504   & 0.134728   & True  & 0.0 & 10 & 90 & 2 \\
0.085898    & 0.031732    & 0.140063   & True  & 0.0 & 30 & 70 & 2 \\
0.103235    & 0.046921    & 0.159549   & True  & 0.0 & 50 & 50 & 2 \\
0.112833    & 0.030839    & 0.194826   & True  & 0.4 & 10 & 90 & 2 \\
0.095228    & -0.006859   & 0.197315   & True  & 0.4 & 30 & 70 & 2 \\
0.149218    & 0.056437    & 0.241998   & True  & 0.4 & 50 & 50 & 2 \\
0.042195    & -0.048001   & 0.132391   & True  & 0.8 & 10 & 90 & 2 \\
0.093129    & 0.009514    & 0.176744   & True  & 0.8 & 30 & 70 & 2 \\
0.105378    & 0.035148    & 0.175607   & True  & 0.8 & 50 & 50 & 2 \\
10.687620   & -2.869784   & 24.245024  & False & 0.0 & 10 & 90 & 5 \\
7.282445    & -6.616917   & 21.181807  & False & 0.0 & 30 & 70 & 5 \\
2.874308    & -4.688494   & 10.437109  & False & 0.0 & 50 & 50 & 5 \\
4.071531    & -5.723477   & 13.866538  & False & 0.4 & 10 & 90 & 5 \\
0.956628    & -10.018669  & 11.931925  & False & 0.4 & 30 & 70 & 5 \\
-5.491321   & -18.211299  & 7.228656   & False & 0.4 & 50 & 50 & 5 \\
0.573767    & -7.492679   & 8.640214   & False & 0.8 & 10 & 90 & 5 \\
-3.820528   & -12.886054  & 5.244998   & False & 0.8 & 30 & 70 & 5 \\
-4.917480   & -15.161070  & 5.326109   & False & 0.8 & 50 & 50 & 5 \\
10.507537   & 0.036861    & 20.978213  & True  & 0.0 & 10 & 90 & 5 \\
12.273186   & 3.825430    & 20.720942  & True  & 0.0 & 30 & 70 & 5 \\
18.397351   & 11.233154   & 25.561549  & True  & 0.0 & 50 & 50 & 5 \\
1.784933    & -7.365845   & 10.935710  & True  & 0.4 & 10 & 90 & 5 \\
8.188129    & 1.217410    & 15.158849  & True  & 0.4 & 30 & 70 & 5 \\
20.694907   & 9.166655    & 32.223160  & True  & 0.4 & 50 & 50 & 5 \\
2.638710    & -9.021860   & 14.299280  & True  & 0.8 & 10 & 90 & 5 \\
10.948408   & 3.223581    & 18.673235  & True  & 0.8 & 30 & 70 & 5 \\
17.933160   & 9.614722    & 26.251598  & True  & 0.8 & 50 & 50 & 5 \\
309.723170  & 48.773653   & 570.672686 & False & 0.0 & 10 & 90 & 10 \\
264.708738  & 8.706979    & 520.710497 & False & 0.0 & 30 & 70 & 10 \\
-310.144245 & -633.900151 & 13.611661  & False & 0.0 & 50 & 50 & 10 \\
-57.120902  & -253.024398 & 138.782594 & False & 0.4 & 10 & 90 & 10 \\
-212.141924 & -498.899103 & 74.615254  & False & 0.4 & 30 & 70 & 10 \\
-429.773537 & -786.701764 & -72.845309 & False & 0.4 & 50 & 50 & 10 \\
-133.179844 & -370.851501 & 104.491814 & False & 0.8 & 10 & 90 & 10 \\
-182.821632 & -456.259702 & 90.616438  & False & 0.8 & 30 & 70 & 10 \\
-435.518703 & -791.043397 & -79.994010 & False & 0.8 & 50 & 50 & 10 \\
220.182609  & 11.927906   & 428.437312 & True  & 0.0 & 10 & 90 & 10 \\
287.629645  & 102.303168  & 472.956122 & True  & 0.0 & 30 & 70 & 10 \\
397.739142  & 159.122421  & 636.355863 & True  & 0.0 & 50 & 50 & 10 \\
31.267834   & -172.938839 & 235.474507 & True  & 0.4 & 10 & 90 & 10 \\
112.227812  & -64.333427  & 288.789050 & True  & 0.4 & 30 & 70 & 10 \\
229.049254  & 78.704906   & 379.393601 & True  & 0.4 & 50 & 50 & 10 \\
75.251773   & -214.074304 & 364.577849 & True  & 0.8 & 10 & 90 & 10 \\
147.828473  & -45.398299  & 341.055245 & True  & 0.8 & 30 & 70 & 10 \\
223.758261  & 63.647799   & 383.868723 & True  & 0.8 & 50 & 50 & 10 \\

\bottomrule
\end{longtable}

\begin{longtable}{S S S c >$c<$ >$c<$ >$c<$ S[round-precision=1]}
\label{tab:var_lq_dimensions_appendix}\\
\caption{LQ environment, with horizon = 2, and target policy with $\vtheta = 0$ and $\log \sigma = 0$. Variance reduction in off-policy gradient, expressed as $\Delta \! \Var$ and its 95\% Gaussian confidence interval $(\Delta \! \Var^-,\Delta \! \Var^+)$, with different hyper-parameters and values of LQ dimensions.}\\

\sisetup{ round-mode = places, round-precision = 3 }\\

\hline
\text{$\Delta \! \Var$} & \text{$\Delta \! \Var^-$} & \text{$\Delta \! \Var^+$} & biased & \beta & N_{\mathrm{BPO}} & N_{\mathrm{PG}} & horizon\\
\hline
\endfirsthead
\endlastfoot
    
 -8.339387 & -24.727999 & 8.049225  & False & 0.0 & 10 & 90 & 2 \\
 0.015846  & -0.078860  & 0.110552  & False & 0.0 & 30 & 70 & 2 \\
 -0.084267 & -0.288979  & 0.120445  & False & 0.0 & 50 & 50 & 2 \\
 -0.061526 & -0.197193  & 0.074140  & False & 0.4 & 10 & 90 & 2 \\
 -0.057192 & -0.164759  & 0.050375  & False & 0.4 & 30 & 70 & 2 \\
 -0.104342 & -0.228757  & 0.020073  & False & 0.4 & 50 & 50 & 2 \\
 -0.036944 & -0.159470  & 0.085583  & False & 0.8 & 10 & 90 & 2 \\
 -0.086518 & -0.184832  & 0.011796  & False & 0.8 & 30 & 70 & 2 \\
 -0.203195 & -0.335921  & -0.070469 & False & 0.8 & 50 & 50 & 2 \\
 -0.008285 & -0.214530  & 0.197959  & True  & 0.0 & 10 & 90 & 2 \\
 0.104098  & 0.032116   & 0.176080  & True  & 0.0 & 30 & 70 & 2 \\
 0.238017  & 0.131980   & 0.344053  & True  & 0.0 & 50 & 50 & 2 \\
 0.011235  & -0.095540  & 0.118011  & True  & 0.4 & 10 & 90 & 2 \\
 0.095955  & 0.012872   & 0.179039  & True  & 0.4 & 30 & 70 & 2 \\
 0.127433  & 0.055541   & 0.199325  & True  & 0.4 & 50 & 50 & 2 \\
 0.002206  & -0.080722  & 0.085135  & True  & 0.8 & 10 & 90 & 2 \\
 0.079307  & 0.000681   & 0.157932  & True  & 0.8 & 30 & 70 & 2 \\
 0.125603  & 0.058244   & 0.192963  & True  & 0.8 & 50 & 50 & 2 \\
 0.194184  & -0.083991  & 0.472359  & False & 0.0 & 10 & 90 & 5 \\
 -0.146855 & -0.614440  & 0.320730  & False & 0.0 & 30 & 70 & 5 \\
 -0.177411 & -0.438773  & 0.083951  & False & 0.0 & 50 & 50 & 5 \\
 -0.289803 & -0.550181  & -0.029424 & False & 0.4 & 10 & 90 & 5 \\
 -0.255346 & -0.520408  & 0.009716  & False & 0.4 & 30 & 70 & 5 \\
 -0.269124 & -0.526112  & -0.012137 & False & 0.4 & 50 & 50 & 5 \\
 -0.129578 & -0.334713  & 0.075557  & False & 0.8 & 10 & 90 & 5 \\
 -0.123404 & -0.340821  & 0.094013  & False & 0.8 & 30 & 70 & 5 \\
 -0.437729 & -0.671352  & -0.204107 & False & 0.8 & 50 & 50 & 5 \\
 -0.834077 & -2.383864  & 0.715710  & True  & 0.0 & 10 & 90 & 5 \\
 0.182321  & -0.092779  & 0.457422  & True  & 0.0 & 30 & 70 & 5 \\
 0.163729  & -0.067104  & 0.394563  & True  & 0.0 & 50 & 50 & 5 \\
 -0.229281 & -0.510593  & 0.052031  & True  & 0.4 & 10 & 90 & 5 \\
 0.088913  & -0.137086  & 0.314913  & True  & 0.4 & 30 & 70 & 5 \\
 0.225710  & 0.014423   & 0.436998  & True  & 0.4 & 50 & 50 & 5 \\
 -0.046998 & -0.214187  & 0.120191  & True  & 0.8 & 10 & 90 & 5 \\
 0.090860  & -0.086864  & 0.268584  & True  & 0.8 & 30 & 70 & 5 \\
 0.229097  & 0.034306   & 0.423888  & True  & 0.8 & 50 & 50 & 5 \\
 1.044491  & 0.832316   & 1.256666  & False & 0.0 & 10 & 90 & 10 \\
 0.040743  & -0.419189  & 0.500674  & False & 0.0 & 30 & 70 & 10 \\
 -0.638193 & -1.225117  & -0.051268 & False & 0.0 & 50 & 50 & 10 \\
 -0.692391 & -1.118963  & -0.265820 & False & 0.4 & 10 & 90 & 10 \\
 -0.385588 & -0.904039  & 0.132862  & False & 0.4 & 30 & 70 & 10 \\
 -0.746861 & -1.588713  & 0.094990  & False & 0.4 & 50 & 50 & 10 \\
 -0.007001 & -0.385542  & 0.371541  & False & 0.8 & 10 & 90 & 10 \\
 -0.372875 & -0.864685  & 0.118934  & False & 0.8 & 30 & 70 & 10 \\
 -0.936066 & -1.681347  & -0.190786 & False & 0.8 & 50 & 50 & 10 \\
 -1.728083 & -5.132161  & 1.675995  & True  & 0.0 & 30 & 70 & 10 \\
 0.268508  & -0.029918  & 0.566934  & True  & 0.0 & 50 & 50 & 10 \\
 -0.118744 & -0.583906  & 0.346419  & True  & 0.4 & 10 & 90 & 10 \\
 -0.272643 & -0.906404  & 0.361118  & True  & 0.4 & 30 & 70 & 10 \\
 0.130968  & -0.137975  & 0.399911  & True  & 0.4 & 50 & 50 & 10 \\
 0.194654  & -0.199835  & 0.589142  & True  & 0.8 & 10 & 90 & 10 \\
 0.306706  & -0.081120  & 0.694532  & True  & 0.8 & 30 & 70 & 10 \\
 0.601394  & 0.316451   & 0.886338  & True  & 0.8 & 50 & 50 & 10 \\

\bottomrule
\end{longtable}



\end{document}